\documentclass[final,onefignum,onetabnum]{siamonline250211}

\usepackage{lipsum}
\usepackage{amsfonts}
\usepackage{graphicx}

\usepackage{epstopdf}
\usepackage{algorithm}

\ifpdf
  \DeclareGraphicsExtensions{.eps,.pdf,.png,.jpg}
\else
  \DeclareGraphicsExtensions{.eps}
\fi

\usepackage{enumitem}
\setlist[enumerate]{leftmargin=.5in}
\setlist[itemize]{leftmargin=.5in}

\newcommand\flarename{FLARE MCMC}

\newsiamremark{remark}{Remark}
\newsiamremark{hypothesis}{Hypothesis}
\crefname{hypothesis}{Hypothesis}{Hypotheses}
\newsiamthm{claim}{Claim}
\newsiamremark{fact}{Fact}
\crefname{fact}{Fact}{Facts}

\headers{\flarename: Fidelity-based Layer-Adaptive REcursive proposals for MCMC}
{H. Venkatesan, C. Shelton, M. Ho, S. Bird, and M. Wu}

\title{FLARE MCMC: Fidelity-based Layer-Adaptive REcursive proposals for MCMC
\thanks{This is the author's accepted manuscript of an article published in
SIAM Journal on Uncertainty Quantification. The final version is available at \url{https://doi.org/10.1137/25M1795194}.
\funding{This work was funded by the U.S. National Science Foundation (NSF) under Grant No. IIS-2435579. SB was supported by NASA-80NSSC21K1840.}}}

\author{Harini Venkatesan\thanks{Department of Computer Science and Engineering, University of California, Riverside 
  (\email{hvenk001@ucr.edu}, \email{cshelton@cs.ucr.edu}, \email{mwu171@ucr.edu})}
  \and Christian Shelton\footnotemark[1]
\and Ming-Feng Ho\thanks{Leinweber Center for Theoretical Physics, University of Michigan (\email{mfho@umich.edu})}
  \and Simeon Bird\thanks{Department of Physics and Astronomy, University of California, Riverside  (\email{sbird@ucr.edu})}
\and Mengxuan Wu\footnotemark[1]}

\usepackage{amsopn}

\usepackage{blindtext}
\usepackage[american]{babel}
\usepackage{amsmath}
\usepackage{amssymb}
\usepackage{mathtools}
\usepackage{xcolor}
\usepackage{subfig}
\usepackage{xfrac}
\usepackage{booktabs} 
\usepackage{subfig}
\usepackage[noend]{algpseudocode}
\usepackage{stmaryrd}
\usepackage{pgf}
\usepackage{float}
\usepackage{siunitx}
\usepackage{etoolbox}
\usepackage{enumitem}
\usepackage{multirow}

\usepackage{tabularx} 

\usepackage{derivative}
\usepackage{numprint}
\usepackage{wrapfig} 
\usepackage[normalem]{ulem}
\definecolor{figgray}{gray}{0.41}

\usepackage{hyperref}

\usepackage{cleveref}
\newcommand{\capJ}{{\mathchoice{}{}{\scriptscriptstyle}{}J}}
\newcommand\smoothder[1]{\frac{\partial}{\partial \smooth{#1}}}
\newcommand\obj[1]{H_{#1}}  

\newcommand\fulltrans[2]{#1\!\shortrightarrow\!#2}
\newcommand\trans[1]{\fulltrans{#1}{\tilde #1}}

\newcommand\fullprop[3]{q_{#1}(#2| #3)}
\newcommand\prop[1]{\fullprop{}{\tilde #1}{#1}}
\newcommand\propback[1]{\fullprop{}{#1}{\tilde #1}}
\newcommand\proplvl[2]{\fullprop{#1}{\tilde #2}{#2}}
\newcommand\proplvlback[2]{\fullprop{#1}{#2}{\tilde #2}}

\newcommand\transp[4]{p_{#1}^{#2}({#3 \rightarrow #4})}
\newcommand\data{D}

\newcommand\fullpimod[2]{\tilde{\pi}_{#1}(#2)}

\newcommand\fullpi[2]{\pi_{#1}(#2)}

\newcommand\smooth[1]{\omega_{#1}}

\newcommand\psip[2]{\psi_{#1}(#2)}
\newcommand\zetap[2]{\zeta_{#1}(#2)}
\DeclareMathOperator*{\E}{\mathbb{E}}

\usepackage{lastpage}

\ifpdf
\hypersetup{
  pdftitle={\flarename},
  pdfauthor={H, Venkatesan, C. Shelton, M. Ho, S. Bird, M. Wu}
}
\fi

\begin{document}

\maketitle
\begin{abstract}
Markov chain Monte Carlo (MCMC) requires only the ability to evaluate the
likelihood, making it a common technique for inference in complex
models. However, it can have a slow mixing rate, requiring the generation of many samples to obtain good estimates and an overall high
computational cost. \flarename{} is a multi-fidelity layered MCMC method
that exploits lower-fidelity approximations of the true
likelihood calculation to improve mixing and leads to overall faster
performance. Such lower-fidelity likelihoods are commonly available in
scientific and engineering applications where the model involves a
simulation whose resolution or accuracy can be tuned. Our technique uses
recursive, layered chains with simple layer tuning; it does not require the likelihood to take any form or
have any particular internal mathematical structure. We demonstrate
experimentally that \flarename\ achieves larger effective sample sizes for the
same computational time across different scientific domains including
hydrology and cosmology.
\end{abstract}



\section{Introduction}

Markov chain Monte Carlo (MCMC) is a workhorse of scientific and
engineering computation.  Most frequently, it is employed to compute the
posterior distribution of model parameters, based on observations.  The
calculated distributions (as represented by samples) give estimates that
can be used in calibration and uncertainty quantification to aid in the
generation of new scientific experiments, clarify the observability of the model,
and resolve scientific theories.

Among the many MCMC algorithms, Metropolis-Hastings MCMC (MH-MCMC) is popular because of its ability to sample from almost any distribution while requiring only the ability to evaluate the
model's likelihood given a parameter setting. Yet, this is also
its weakness, as it has no additional knowledge of the problem setting to
guide its sampling effectively.  Therefore, its mixing time (speed of
generating effectively new samples) can be slow and the overall
algorithm computationally burdensome.

Methods such as Hamiltonian Monte Carlo and its variants
\cite{Duaetal87,Nea96,HofGel14} speed up mixing by adding auxiliary momentum
variables, allowing longer steps to reduce correlations between consecutive
samples. Such methods require computing the gradient of the log target
distribution with respect to the parameters, something that could be
prohibitively expensive when the distribution is evaluated through lengthy simulation code. For instance, the cosmological simulation we use in our experimental results that aims to approximate the posterior density conditioned on the galaxy power spectrum from SDSS-III Baryon Oscillation Spectroscopic Survey (BOSS) Data \cite{BOSS:2013,BOSS:2017} cannot be modified to produce gradients. Due to the complexity and non-differentiability of the forward cosmological simulation, gradients with respect to the model parameters are not available. Therefore, methods like auto-differentiation cannot be applied, nor is there an analytic form for the gradients, prohibiting the use of gradient-based inference methods.

\flarename{}
speeds up the mixing time
of MH-MCMC by exploiting lower-fidelity models of the same problem. Many
engineering or scientific computational models can be run at multiple
fidelities. 
\flarename{} exploits a set of computationally cheaper posterior calculations, each an approximation of the true posterior.  Many  posteriors involve solving a PDE, ODE, or integral.  For these, coarsening the spatial or temporal grid leads to cheaper approximations.  For those with constraint or optimization solvers, reducing the solvers' tolerances or maximum number of iterations can similarly lead to cheaper approximations.  We further show a physics example where the underlying simulation can be coarsened by reducing the number of representative particles.

By recursively employing MCMC chains, we can use the coarser resolution models
to guide the higher resolution MCMC chain.  The result is a sampler for
the target model that converges faster and generates
more effective samples per computation time, even considering the extra
time necessary to employ the lower-fidelity computations.

\section{Background}

Markov chain Monte Carlo is a class of algorithms designed to sample from a complicated target distribution
by constructing an easy-to-simulate Markov chain such that the stationary distribution of
the Markov chain is the target distribution.
Commonly, this target distribution is the posterior distribution of a set of parameters, conditioned
on observations.  Let $\data$ be the observations
and $\theta \in \Theta \subset \mathbb{R}^R$ be the parameters.
Assuming a prior distribution on the parameters $p(\theta)$,
the target posterior distribution of interest, $\pi(\theta\mid\data)$, is obtained through Bayes' theorem:
\begin{equation}
\label{eqn:bayes}
	\pi(\theta\mid\data) = \frac{\mathcal{L}(\data\mid \theta)p(\theta)}{p(\data)} \propto \mathcal{L}(\data\mid \theta)p(\theta)
\end{equation}
where $\mathcal{L}(\data\mid \theta)$ is the likelihood of the data, which in many scientific
applications requires a lengthy simulation to evaluate.
We only require the ability to evaluate $\pi(\theta\mid\data)$ up to a constant of proportionality,
and therefore the denominator of $p(\data)$ is safely ignored.
That $\pi(\theta\mid\data)$ is a conditional distribution is largely irrelevant for MCMC, so we will
just let $\pi(\theta)$ denote the distribution of interest (equal to $\pi(\theta\mid\data)$ if the 
underlying distribution is a posterior, but it could be any distribution over $\theta$).

\subsection{Metropolis-Hastings MCMC}
We begin by focusing on the Metropolis-Hastings method for Markov Chain Monte Carlo (MH-MCMC) introduced by Hastings (1970) \cite{MHMCMC1}.
The $(i\!\!+\!\!1)$th sample, $\theta^{i+1}$, is generated based on the previous sample in the
chain, $\theta^i$, in a two-step process.
First, a proposed next state, $\tilde\theta^i$ is generated from a proposal distribution,
$\prop{\theta^i}$.  
Then, $\tilde\theta^i$ is either accepted or rejected as $\theta^{i+1}$
according to a carefully constructed acceptance probability.  If accepted,
$\theta^{i+1}\!=\!\tilde\theta^i$, otherwise $\theta^{i+1}\!=\!\theta^i$.
Often, a normal distribution centered at $\theta^i$ is used as
the proposal distribution $\prop{\theta^i}$, but almost any proposal distribution can be used,
subject to mild conditions (for instance, that $\prop{\theta^i}$ is positive everywhere).  With a chosen $\prop{\theta^i}$,
the acceptance probability, $\mathcal{A}$, for the transition $\trans{\theta^i}$ is
\begin{equation}
\label{eqn:accept}
	\mathcal{A}(\trans{\theta^i}) = \min(1,r(\trans{\theta^i}))
\end{equation}
where
\begin{equation}
\label{eqn:rmin}
	r(\trans{\theta^i}) = \frac{\pi(\tilde{\theta}^i)}{\pi(\theta^i)}\frac{\propback{\theta^i}}
							{\prop{\theta^i}}\,\,.
\end{equation}

Although the standard Metropolis-Hastings MCMC algorithm can be an easy way to
sample from a posterior distribution, it requires sufficient samples to be
an effective approximation of the posterior distribution.  When the chain
is slow to mix (due to a less-than-optimal proposal distribution),
consecutive samples are highly dependent and more samples must be taken to
achieve a set representative of the true distribution.  When the evaluation
of $\pi(\theta^i)$ (necessary for the calculation of
Equation~\ref{eqn:rmin}) is computationally expensive, this is particularly
problematic.

\subsection{Related Work} 

Our goal of accelerating MCMC sampling is shared by a large body of work. These approaches involve methods that couple chains (like simulated tempering), methods that aim to reduce the variance of estimators for a target using cheap approximations from multiple fidelities (like MLMC), and methods that use cheap models to build MCMC proposals.

Like \flarename,
methods such as simulated tempering and coupled MCMC \cite{MCSimulation, Marinari, Altekar_Dwarkadas_Huelsenbeck_Ronquist_2004} use multiple chains.  Samples are accepted or rejected by evaluating the energy of the process and adjusting the temperature of the model. Two chains are run in parallel at different temperatures, and the system swaps between different temperatures. Reversible jump MCMC \cite{10.1093/biomet/82.4.711, Al-Awadhi} also jumps between chains (of different dimensions). While \flarename\ shares the notion of multiple chains, because it solves a different problem (to take advantage of simulations that are orders of magnitude cheaper to evaluate), the resulting structure is very different. 
Methods such as sequential MCMC or particle filtering \cite{smc, Doucet2001} use the notion of approximations of the target by a large number of samples called particles that are propagated across time using importance sampling. However, those are filtering frameworks and do not converge to a stationary distribution. Thus, though appearing related in its structure, \flarename\ is quite different to these methods. 

A highly influential body of work focuses on reducing the variance of the final estimator for a target expectation in the Multilevel Monte Carlo (MLMC) framework. Taking inspiration from the multilevel Monte Carlo method \cite{MLMC} for high-dimensional, parameter-dependent integrals and Multilevel Monte Carlo Path Simulation \cite{Giles-MLMC}, Hoang et al. \cite{hoang2013complexity} proposed a multilevel MCMC method that applied to Bayesian Inverse problems.

The core idea is to decompose a high-fidelity expectation into a telescoping sum using a hierarchy of computational models with increasing model resolution. This method achieves computational speedup by estimating the low variance difference terms with small number of samples, while the bulk of the computational efforts is spent on the cheap, low-fidelity estimator. The authors also provide rigorous complexity proof, showing how quickly posterior expectation might converge when running iterative samplers on sparse grids using telescopic expansion on the discretization error. 

This MLMC framework has been extended in many directions. Multilevel sequential Monte Carlo samplers \cite{beskos2017multilevel} and Multilevel Particle Filters \cite{jasra2017multilevel} along with previous work \cite{hoel2016multilevel, gregory2016multilevel, gregory2017seamless} extended MLHC to sequential Monte Carlo. Jasra et al. \cite{jasra2018bayesian} extended MLHC to the problem of static parameter estimation in partially observed diffusions. Problems with multiple ways of discretizing were addressed by Multi-Index MCMC \cite{haji2016multi, jasra2018multi}. These MLMC methods use samples from all chains in a telescoping estimator. They target the MSE of a specific quantity of interest.  In contrast, FLARE MCMC uses only samples from the finest chain (like methods discussed below) and targets the chain's mixing time, rather than MSE. Our theoretical analyses in this paper focus on the ergodicity and convergence rates for FLARE MCMC and are not specific to any particular problem domain.

More similar to \flarename, several previous methods have shown that replacing the proposal with an
approximation with generally high acceptance probability reduces the
computational cost of the standard Metropolis-Hastings algorithm
significantly. This idea was first proposed by Christen and Fox \cite{DAMCMC, Fox1997SamplingCI} as a two-stage
MCMC method that tests the original proposal using a cheap approximation to
find moves in the chain that are more likely to be accepted. In other
words, a candidate is accepted with the likelihood of the approximate model
before it is evaluated with the more expensive model. In preconditioned MCMC using coarse-scale simulation proposed by Efendiev, Hou and Luo  \cite{doi:10.1137/050628568}, two-stages are used to reduce the computational cost incurred in the fine fidelity by testing the coarse model based on high-fidelity multiscale finite volume model. However, this only
performs a single check with a cheap approximation, and does not exploit it to run a full MCMC subchain.

Multilevel Markov chain Monte Carlo (MLMCMC) \cite{MLMCMC} achieves computationl efficiency on the finer levels. If the coarse proposal from the approximation is rejected by the fine level, the coarse
chain continues independently of the fine chain instead of recursively
starting the next coarse chain from the current sample of the fine chain.
MLMCMC uses a user-specified variable that is internal to the likelihood
computation and shared across the levels (for instance, the predicted
observations to be compared with the true observations through a noise
model).
The samples drawn from the coarse approximation are
used to reduce the variance of this internal variable achieving
better proposals from the coarse fidelities.

Lykkegaard et al. \cite{doi:10.1137/22M1476770} proposed Adaptive Multilevel Delayed Acceptance (MLDA),
which adapted a recursive version of \mbox{MLMCMC} over multiple levels. Here, the
coarse inner subchain used to generate subsequent proposals for the current chain is initiated from the current sample from the outer chain again instead of independently continuing the fine chain even if coarse proposal is
rejected. MLDA also applies an Adaptive Error Model (AEM) \cite{KAIPIO2007493} to account for
discrepancies between the different fidelities.  It takes the two-level AEM from Adaptive Delayed Acceptance Metropolis Hastings
\cite{AEM, AEM-summary} and extends it by adding a telescoping sum of differences in the model output across multiple levels.

Several multilevel MCMC methods based on delayed rejection, in contrast to delayed acceptance, have
also been proposed and are summarized by Peherstorfer et al. \cite{multifidelitysurvey}. Adaptive methods in multistage MCMC \cite{Tierney1999SomeAM} proposed using an independence sampler that
is a good approximation for the posterior distribution in the
first stage and random walk in the second stage to help with poor approximation by the independence sampler. Delayed rejection in MCMC \cite{reversiblejumpreject} suggested using a normal distribution as the proposal
in the first level and a normal distribution with the same
mean but higher variance in the second level. Higdon et al. \cite{978393} proposed using multiple MCMC chains from low and high fidelities and coupling them using a product chain and "swapping" updates allowing information to move between the two fidelity scales. An accelerated MCMC method using local approximations was developed by Conrad et al. \cite{Conrad_2016} that uses local approximations of either the log-likelihood function or the forward model of different simulations into the Metropolis-Hastings kernel. 
Although
these methods use approximations as proposals, they do not
exploit layered or recursive MCMC chains.

Cai and Adams \cite{NEURIPS2022_8803b9ae} proposed a multi-fidelity Monte Carlo method (MFMC) that uses randomized fidelities as the approximate for the target fidelity.
The algorithm does not converge to the true posterior, but the resulting samples can be used to estimate expectations through a specific ``sign-correction'' formula.  
Our method follows a hierarchy of levels in its sampling while also sampling from the true posterior and provides a simpler alternative to previous multilevel methods.   

Our multi-fidelity layered MCMC algorithm, \flarename, has a similar
structure to MLDA in terms of the recursive layers and achieves a similar
amount of effective samples across multiple chains of MLDA.  However, our
method for mitigating the differences between approximations is simpler in construction and implementation than that of MLDA,
does not require the identification of any internal variables of the distribution to
be sampled, and generates more effective samples in a shorter amount of
time and computational cost.  We demonstrate this on real-world large
scientific problems. We also show theoretical convergence rates, optimal value for the number of inner steps $M$ and prove ergodicity of the adaptation in layer tuning. 

\section{\flarename: Fidelity-based Layer-Adaptive REcursive proposals for MCMC}

We consider a series of models, ordered by fidelity.  For instance, we
might have a model that evaluates a differential equation
numerically as the main part of the likelihood calculation (simulating
forward in time); 
the resolution of the spatial or temporal grid used to evaluate the model can
be tuned to change its fidelity.  The highest fidelity model is our
``true'' model, from whose posterior we wish to sample. 
\flarename{} draws samples from the true model.
Its nested chains
use the coarser fidelity models as cheap approximations of this finest
fidelity model to speed up mixing.

Where the standard Metropolis-Hastings algorithm uses a distribution $q$
that proposes the next sample, 
\flarename\ uses nested Markov chains as the proposal
distribution. In a recursive fashion, each layer uses the result of
another MCMC chain with a coarser approximation as its proposal. The
recent sample in the current chain is the starting sample in the nested
chain. The coarser chain runs for $M$ iterations, with
each proposed sample evaluated by the likelihood of the cheaper 
layer. The last, $M$th, sample of the coarser chain is proposed as the
candidate for the next sample in the current chain.  At the coarsest
fidelity/layer, a standard proposal distribution is used, for instance a
normal distribution centered on the current point.

This avoids numerous expensive likelihood calculations in the fine fidelity
that might end up rejected, and it allows the proposal to generate
samples that are more likely to get accepted by the finest fidelity, since
it was accepted by an approximation already.  While the coarser chains have
their own computational cost, they can often be orders of magnitude faster
to evaluate, thus leading to an overall savings in the running time of the
entire algorithm, as measured by the quality of the samples generated per
computational time.

\begin{figure*}
	{\centering \resizebox{\textwidth}{!}{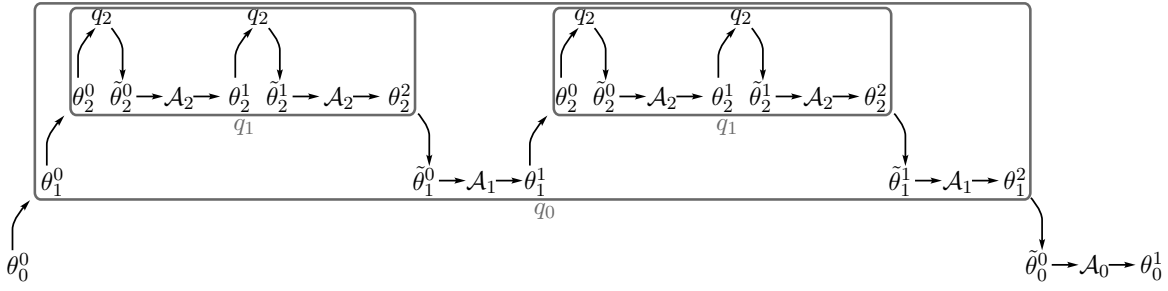}

	}
	\caption{One sampling step from the finest layer 
	with two coarse fidelities 
	and two iterations per nested chain. Refer to text in Section \ref{sec:alg-section}. 
	}
	\label{fig:scheme}
\end{figure*}

\subsection{Algorithm Specification}
\label{sec:alg-section}
Let $\theta \in \mathbb{R}^R$ be the set of
parameters (over which we are sampling) and let $j \in \{ 0,1,...,J\}$ be the
fidelities ordered in a decreasingly complex fashion ($0$ is the 
``true'' model and $J$ is the coarsest fidelity).  We let $\fullpi{j}{\theta}$
be the posterior distribution according to the $j$th fidelity model and
$\proplvl{j}{\theta}$ be the proposal distribution for layer $j$. Here, 
$\theta_j^i$ is the $i$th sample in the current chain at layer $j$.
The goal is to sample from $\fullpi{0}{\theta}$.

\algnewcommand{\IIf}[1]{\State\algorithmicif\ #1\ \algorithmicthen}
\algnewcommand{\ElseIIf}[1]{\algorithmicelse\ #1}
\algnewcommand{\EndIIf}{}

\newcommand\helpname{FLARE-Chain}

\begin{algorithm}
\caption{\helpname ($\theta^0_j$,n,$j$)}\label{alg:MCMCMCreal}
\begin{algorithmic}
\For{$i= 0, \dots, n-1$}
    \If{$j=J$} \hfill\Comment{coarsest layer}
        \State{Sample $\tilde\theta^i_j$ from $\fullprop{j}{\cdot}{\theta^i_j}$}
        \State{Accept $\theta^{i+1}_j = \tilde{\theta}^i_j $ with probability $\mathcal{A}$ from Equation \ref{eqn:accept}}
        \State{Otherwise, reject and $\theta^{i+1}_j = \theta^i_j$}
    \Else  \State{${
    \theta^1_{j+1},\dots,}\theta^{M}_{j+1} = $ \helpname($\theta^{i}_j$, M, $j\!+\!1$)}
        \State{$\tilde\theta^i_j = \theta^M_{j+1}$}
        \State{Accept $\theta^{i+1}_j = \tilde{\theta}^i_j $ with probability $\mathcal{A}_j$ from Equation \ref{eqn:newaccept}}
        \State{Otherwise, reject and $\theta^{i+1}_j = \theta^i_j$}
    \EndIf
\EndFor
\State{\Return $\theta^1_{j}, \dots, \theta^n_{j}$}
\end{algorithmic}
\end{algorithm}

In \flarename, the proposal distribution
$\proplvl{j}{\theta_j^i}$ for iteration $i$ of a chain at 
layer $j$
is another MCMC chain of $M$ steps targeting the (coarser)
posterior $\fullpi{j+1}{\cdot}$, starting this nested chain at $\theta_j^i$. The 
result of $M$ steps using a chain with stationary distribution $\fullpi{j+1}{\cdot}$ is the 
proposal for $\tilde{\theta}_j^i$: $\proplvl{j}{\theta_j^i}$.
More algorithmically, to generate $\tilde\theta^i_j$ from $\theta^i_j$, we run
the (coarser) MCMC algorithm
at layer $j+1$.  We start with $\theta^0_{j+1} = \theta^i_j$ and continue
the coarser MCMC sampler
until $\theta^M_{j+1}$.  We then set $\tilde\theta^i_j = \theta^M_{j+1}$.
At the coarsest layer, $\proplvl{J}{\theta^i}$ is a standard
simple proposal distribution.
Figure~\ref{fig:scheme} pictorially demonstrates this for $J=2$ inner layers, each with $M=2$ steps.

With the sampling scheme so defined, it remains to construct the acceptance probability
for each layer: $\mathcal{A}_0, \mathcal{A}_1, \dots, \mathcal{A}_J$.
We follow a standard Metropolis-Hastings method for every layer and therefore
$\mathcal{A}_j = \min(1,r_j(\trans{\theta_j^i}))$.  At the coarsest layer,
the ratio $r_J(\trans{\theta_J^i})$ is just as in Equation~\ref{eqn:rmin}
because $q_J$ is a standard proposal distribution.

When $j<J$, the proposal distribution is from a Markov chain that obeys detailed balance.  Therefore
\begin{equation}
\frac{\proplvlback{j+1}{\theta_j^i}}{\proplvl{j+1}{\theta_j^i}} =
\frac{\fullpi{j+1}{\theta_j^i}}{\fullpi{j+1}{\tilde\theta_j^i}}\quad\quad 0\le j<J
\end{equation}
and thus
\begin{equation}
\label{eqn:newaccept}
	\mathcal{A}_j(\trans{\theta^i_j}) = \min\left(1,
	\frac{\fullpi{j}{\tilde\theta^i_j}}{\fullpi{j}{\theta^i_j}} \cdot
	\frac{\fullpi{j+1}{\theta^i_j}}{\fullpi{j+1}{\tilde\theta^i_j}}
	\right)\,\,.
\end{equation}
Note this equation does not depend on $M$ (the number of steps for the
coarser chain at layer $j+1$).  While this chain has almost certainly
not mixed for small $M$, the ratio
$\sfrac{\proplvlback{j}{\theta_j^i}}{\proplvl{j}{\theta_j^i}}$ is the same
as if the chain had completely mixed and the proposed new state,
$\tilde\theta_j^i$, were from the true posterior of the model at layer $j+1$.  The values $\fullpi{j+1}{\tilde\theta^i_{j}}$ and
$\fullpi{j+1}{\theta^i_{j}}$ were already calculated as part of the
chain at layer $j+1$ and therefore do not take any additional computation
time.
\flarename\ is summarized in Algorithm \ref{alg:MCMCMCreal}. To gather $N$
samples from the true posterior, the algorithm is called with 
\helpname$(\theta^0, N, J = 0)$.

\subsection{Convergence Rate}
We show a convergence rate for \flarename{}. We measure the distance to the stationary in terms of total variation distance as follows.

\begin{definition}[Strasser (1985) \cite{Strasser+1985}]
\label{def:tv-defn}
The total variation distance between two probability measures $\nu_1$ and $\nu_2$ is defined as
\[
\|\nu_1 - \nu_2\| = \sup_{A} |\nu_1(A) - \nu_2(A)|.
\]
\end{definition}

The minorization condition of Markov chains, used by Roberts and Rosenthal \cite{Roberts_2004}, provides a means of bounding the convergence rate.
For any Markov chain with a one step transition probability of $\transp{}{}{\theta^i}{\theta^{i+1}}$, we let $\transp{}{n}{\theta^i}{\theta^{i+n}}$ denote the corresponding $n$ step transition probability. Formally, the definition of the minorization condition is stated below.

\begin{definition}[Roberts and Rosenthal (2004) \cite{Roberts_2004}]
\label{def:minorization-cond}
    A Markov chain on $\Theta$ satisfies the minorization condition if there exists an $\epsilon>0$, a positive integer $n$, and a probability measure $\nu(.)$ such that 
    \begin{equation}
         \transp{}{n}{\theta^0}{\theta^{n}} \geq \epsilon \nu(\theta^{n}) \quad\quad  \forall \theta^0, \theta^{n} \in \Theta.
    \end{equation}
\end{definition}

 With respect to the stationary distribution, the probability of transitioning from $\theta$ to $\theta^\prime$ can be minorized by a lower bound such that $\transp{j}{1}{\theta}{\theta^\prime} \geq \epsilon \fullpi{j}{\theta^\prime}$.

\begin{lemma}
\label{lem:minlowerbound}
Assume the minorization condition holds at the innermost level ($j = J$):
$\transp{\capJ}{1}{\theta^i_{\capJ}}{\theta^{i+1}_{\capJ}} \geq \xi_J \cdot \fullpi{\capJ}{\theta^{i+1}_J}$
for some $\xi_{\capJ} > 0$. Then, there exists a minorized lower bound on levels $j < J$ such that \begin{equation}
    \transp{j}{1}{\theta^i_j}{\theta^{i+1}_j} \geq \xi_j \cdot \fullpi{j}{\theta^{i+1}_j}
    \quad\quad\forall \theta^i_j, \theta^{i+1}_j
\end{equation}
where $\xi_j = (1 - (1 - \xi_{j+1})^M) \cdot \min\limits_{\theta} \left( \frac{\pi_{j+1}(\theta)}{\pi_j(\theta)} \right)$.  
\end{lemma}

 We call $\xi_j$ the minorization constant for level $j$.
 This satisfies the necessary minorization condition of Theorem 8 in the original paper \cite{Roberts_2004}. This allows us to get a quantitative bound on the distance to the stationary of every level as stated in Theorem \ref{thm: conv-thm}. The proof for Lemma \ref{lem:minlowerbound} can be found in Appendix \ref{proof:app-conv-rates}.

\begin{theorem}
\label{thm: conv-thm}
    Let $\transp{j}{n}{\theta^0_j}{\cdot}$ be the distribution for layer $j$ with an invariant target probability $\pi_j(\cdot)$. \flarename\ is uniformly ergodic
    and converges as $\left\|\transp{j}{n}{\theta^0_j}{\cdot}  - \fullpi{j}{\cdot} \right\| \leq (1-\xi_j)^n$ where $\xi_j = (1 - (1 - \xi_{j+1})^M) \cdot \min\limits_{\theta} \left( \frac{\pi_{j+1}(\theta)}{\pi_j(\theta)} \right)$ as in Lemma \ref{lem:minlowerbound}. Most critically, it holds for layer $j = 0$. 
\end{theorem}

\begin{proof}
    The results follow from the minorization condition established in Lemma \ref{lem:minlowerbound}. 
\end{proof}

Although the theorem above establishes the convergence of the chain to the invariant target distribution, it illustrates several aspects about the inner chains. In particular, ergodicity requires that the chain be able to reach all regions of the target's support. Therefore, the coarser approximations' supports must be supersets of the finer ones.
The coupling strength, $\xi_j$, in turn depends on the coupling strength of the coarser
approximations' chains, $\xi_k, k>j$.  Thus, the effects of the mixing times of the inner
chains on the outer chain are captured in this theorem.

The similarities of the approximations to each other are captured in the $\min_\theta\left(\frac{\pi_{j+1}(\theta)}{\pi_j(\theta)} \right)$ terms. 
Thus, the theorem also quantifies
the effects of similarities and dissimilarities between the approximations on the total convergence rate.  If the modes of the target distribution are preserved across coarsening, then
we would expect these terms to be larger and therefore the outer chain to mix faster.
Ideally, the minimum in this term could be replaced with an expectation, thus turning it
into the KL-divergence between adjacent layers.  We have not yet determined whether or how this
might be possible.

\subsection{Optimal number of inner steps $M$}

To understand how to select the number of inner steps for each layer $M_j$, we derive a theoretical expression for optimal $M_j$ that balances the decoupling rate of the chains and the cost of likelihood evaluations. The following lemma provides an analytical expression for the value of $M_j$ that maximizes the cost-aware decoupling rate. The proof for Lemma \ref{lem:optimal_M} can be found in Appendix \ref{proof:app-optM}. 

\begin{lemma}
\label{lem:optimal_M}

For layers $0 \leq j < J$, suppose the minorization condition hold such that 
$\transp{j}{1}{\theta^i_j}{\theta^{i+1}_j} \geq \xi_j \cdot \fullpi{j}{\theta^{i+1}_j}$,
where $\xi_j = \left(1 - (1 - \xi_{j+1})^{M_j} \right) \cdot \min\limits_{\theta} \left( \frac{\pi_{j+1}(\theta)}{\pi_j(\theta)} \right)$ is the minorization constant, and $M_j \geq 0$ is the number of inner steps. 
Let the total cost per step at level $j$ be $B_j = b_j + M_j \cdot B_{j+1}$ where $b_j > 0$ is the cost of evaluating the likelihood at level $j$ and $B_{j+1} > 0$ is the cost of a single evaluation of the inner layer $j+1$. Define the cost-aware computational decoupling rate as
\begin{equation}
    f(M_j) = \frac{ \left(1 - (1 - \xi_{j+1})^{M_j} \right) \cdot \min\limits_{\theta} \left( \frac{\pi_{j+1}(\theta)}{\pi_j(\theta)} \right)}{B_j}.
\end{equation}

    Then the real-valued maximizer $M_j^*$ of $f(M_j)$ is 
    \begin{equation}
            M_j^* = -\frac{1}{\Upsilon}W_{-1}\big(-e^{-\Upsilon\mu}\big) - \mu
    \end{equation}

    where $\Upsilon = -\log(1-\xi_{j+1})$, $\mu = \frac{b_j}{B_{j+1}} + \frac{1}{\Upsilon}$ and $W_{-1}$ is $-1$ branch of the Lambert $W$ function.
\end{lemma}

 However, this expression is not directly usable in practice, since it depends on the unknown coupling minorization constant $\xi_{j+1}$ which is generally unknown and difficult to approximate in MCMC settings. The optimizer $M^*_j$ is a real-valued quantity, whereas in practice the number of inner steps must be an integer. Instead, we have found empirical values for $M$ that offers the best trade-off between computation time and sampling efficiency in the experiment section. Nevertheless, this lemma provides a theoretical benchmark for the optimal trade-off between computational cost of likelihoods and effective mixing across layers.

\newcommand\figwidth{0.6\textwidth}

\begin{wrapfigure}{r}{\figwidth}
	\resizebox{\linewidth}{!}{\input{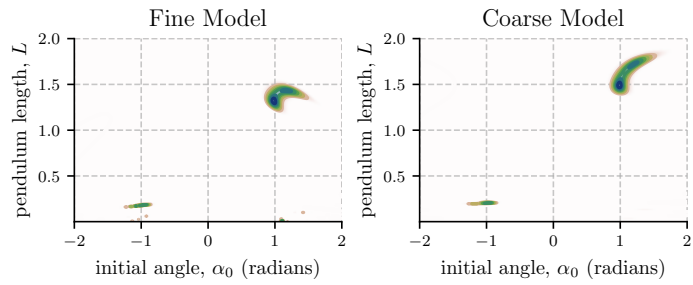}}
\caption{Posterior of different fidelities; coarse model is the small angle
	approximation.  See Section~\ref{ss:pendulum}.}
\label{fig:post}
\end{wrapfigure}

\subsection{Layer Tuning}

The algorithm above uses the coarser fidelities to guide the finer ones.  
Early in the chain, this is useful for quickly driving the samples toward
high-probability regions.
However, this mismatch between the fidelities can cause problems later
because it can steer the chain away from high-probability regions in the fine
fidelity model that do not overlap with high-probability regions of the
coarse fidelity model.  Figure~\ref{fig:post} demonstrates an example of
such partial, but not complete, overlap in one of our examples. To combat this, we present a simple modification that does not require
estimation of any internal variables of the probability models, nor 
estimation of means or variances from multiple chains. 

Recall $\fullpi{j}{\theta_j}$ is known up to a normalizing constant:
$    \fullpi{j}{\theta_j} = \sfrac{\fullpimod{j}{\theta_j}}{Z_j}$
where $\fullpimod{j}{\theta_j}$ is the unnormalized distribution and $Z_j$ is the normalzing constant. We modify the target distributions for coarser chains (and thus the proposal distributions for all $j>0$) as
\begin{align}
\psip{j}{\theta_j} &=  \sfrac{\big(\fullpimod{j}{\theta_j} + \smooth{j}\big)}{\zetap{j}{\smooth{j}}}
		& \forall\ \ 0\!<\!j\!\le\! J
	\label{eqn:adapt}
\end{align}
where $\zetap{j}{\smooth{j}}$ is the normalizing constant of this new distribution which depends on $\smooth{j}$.\footnote{We assume the domain of $\theta$, $\Theta$, is of finite volume.}

We now use $\psip{j}{\theta_j}$ in
place of $\fullpi{j}{\theta_j}$ in Equation~\ref{eqn:newaccept}, therefore modifying the acceptance ratio for all $\ 0\!<\!j\!\le\! J\,\, $ as
\begin{align}
    \mathcal{A}_j(\trans{\theta^i_j}) &= \min\left(1,
	\frac{\psip{j}{\tilde\theta^i_j}}{\psip{j}{\theta^i_j}} \cdot
	\frac{\psip{j+1}{\theta^i_{j+1}}}{\psip{j+1}{\tilde\theta^i_{j+1}}}
	\right) \,\,.
 \label{eqn:adaptaccept}
\end{align}
For the finest layer, things remain the same (or alternatively, $\smooth{0} = 0$), because we do not want to change the distribution of the overall sampler.

This effectively mixes the stationary distribution of the $j$th layer
with a uniform distribution (we have added a constant to the posterior and then renormalized), encouraging the proposal to explore more widely
than the coarser layer would normally.  While unsophisticated, we found it
simpler to implement and compute than other options and just as effective.

Instead of leaving $(\smooth{1}, \smooth{2}, \dots, \smooth{J})$ as hyper-parameters, we use gradient descent to adapt them over the course of the sampling.  We adjust $\smooth{j+1}$ to minimize the Kullback-Leibler divergence between layers $\psi_j$ and $\psi_{j+1}$: 
\begin{equation}
    \text{KL}(\psi_{j} \Vert \psi_{j+1}) = \E_{\theta \sim \psi_j}[\ln(\psi_{j})] - \E_{\theta \sim \psi_j}[\ln(\psi_{j+1})]\,\,.
\end{equation}

This tries to make the coarser (approximating) distribution $\psi_{j+1}$ more similar to the distribution $\psi_j$.
Because the first term does not depend on $\smooth{j+1}$, the objective function is to maximize 
\begin{equation}
    \obj{j+1} = \E_{\theta \sim \psi_j}\Bigl[\ln(\psi_{j+1})\Bigr] \,\,.
\end{equation}
Using Equation~\ref{eqn:adapt}, 
\begin{align}
    \smoothder{j+1} \obj{j+1}
    &= 
    \smoothder{j+1}\left(\E_{\theta \sim \psi_j}\left[\ln\left(\fullpimod{j+1}{\theta} + \smooth{j+1}\right)\right] - \ln \zetap{j+1}{\smooth{j+1}}\right) \nonumber
    \\
    &= \E_{\theta \sim \psi_j}\left[\smoothder{j+1}\ln\left(\fullpimod{j+1}{\theta} + \smooth{j+1}\right)\right] 
    - \E_{\theta \sim \psi_{j+1}}\left[\smoothder{j+1}\ln\left(\fullpimod{j+1}{\theta} + \smooth{j+1}\right)\right]
    \label{eqn:derivJ}
\end{align}
where the second step replaces the derivative of the log-partition function with the expected derivative of the log-probability. 

The first term is an expectation with respect to the distribution at the lower layer $j$.  We assume that the lower layer
has mixed and therefore, the starting state for the chain at layer $j+1$ is a sample drawn from $\psi_j$.  The second term is an expectation with respect to the distribution at this layer, $j+1$.  We let the sample at the \emph{end} of this chain after $M$ steps approximate a sample from
this distribution.  This is similar to the approximation employed by $M$-step contrastive divergence \cite{hinton2002training}.  Although this is not guaranteed to converge \cite{pmlr-v9-sutskever10a}, in practice we have found it to work well.  Thus, the total derivative for the gradient ascent update is
\begin{equation}
\label{eqn:grad_descent_omega}
    \smoothder{j+1}\obj{j+1} \approx \frac{1}{\fullpimod{j+1}{\theta_{j+1}^0} + \smooth{j+1}} - \frac{1}{\fullpimod{j+1}{\theta_{j+1}^M} + \smooth{j+1}}\,\,.
\end{equation}
Note that these denominators are calculated during the MCMC chain and therefore the derivative requires very little extra computation.

A single $\smooth{j}$ is kept for each layer and is maintained across 
subchains at that layer. We use a learning rate of $10^{-3}$ to adjust $\smooth{j}$ for all
experiments.  An update is made on layer $j$ once after each $M$-step subchain.

\newcommand\tuningname{FLARE-with-layer-tuning}

\begin{algorithm}
\caption{\tuningname ($\theta^0_j$,n,$j$)}
\label{alg:shrekMCMC}
\begin{algorithmic}
\For{$i= 0, \dots, n-1$}
    \If{$j=J$} \hfill\Comment{coarsest layer}
        \State{Sample $\tilde\theta^i_j$ from $\fullprop{j}{\cdot}{\theta^i_j}$}
        \State{Accept $\theta^{i+1}_j = \tilde{\theta}^i_j $ with probability $\mathcal{A}$ from Equation \ref{eqn:accept}}
        \State{Otherwise, reject and $\theta^{i+1}_j = \theta^i_j$}
    \Else  \State{${
    \theta^1_{j+1},\dots,}\theta^{M}_{j+1} = $ \tuningname($\theta^{i}_j$, M, $j\!+\!1$)}
        \State{$\tilde\theta^i_j = \theta^M_{j+1}$}
       \State{Update gradient of $\omega_{j+1}$ using $ \smoothder{j+1}\obj{j+1}$ from Equation \ref{eqn:grad_descent_omega}}
       \State{Accept $\theta^{i+1}_j = \tilde{\theta}^i_j $ with probability $\mathcal{A}_j$ from Equation \ref{eqn:adaptaccept}}
        \State{Otherwise, reject and $\theta^{i+1}_j = \theta^i_j$}
    \EndIf
\EndFor
\State{\Return $\theta^1_{j}, \dots, \theta^n_{j}$}
\end{algorithmic}
\end{algorithm}

 To make the innermost Gaussian proposal more robust, we adaptively update the covariance of the proposal distribution as initially proposed in the AM algorithm \cite{bj/1080222083}.
 We use the history of chains from the coarsest layer $\theta^0_J, \theta^1_J, \dots ,\theta^t_J$ to update the covariance for the inner most proposal distribution. By using all previous states of the coarsest layer, the proposal distribution quickly adapts using the accepted samples. This rapid start of adaptation ensures good mixing in the inner most layer which gives higher quality candidate samples for the finer chains. We show the recursive algorithm with layer tuning adaptation added in Algorithm \ref{alg:shrekMCMC}. Here we update the gradient after $M$ steps of each layer and use it in the acceptance probability with $\omega$ mixed in as a uniform distribution to the target distribution.

\subsection{Ergodicity of Layer Tuning}

We show \flarename{} with adaptive tuning of the proposals at each layer is ergodic.  This can be shown with diminishing adaptation and simultaneous uniform ergodicity. 

\begin{lemma}
\label{lem:layertuninglem}
    For layers $0 \leq j < J$, let $\gamma_j \in \Gamma_j$ be the adaptations for the proposal at layer $j$ or the chain at level $j+1$, i.e, $\gamma_j = \smooth{j+1} \leftrightarrow \psip{j+1}{\theta} \leftrightarrow \transp{j+1}{}{\theta}{\cdot}$ where $\Gamma_j \in \mathbb{R} $ and $\psi_{j+1}$ is the target at layer $j+1$ with the layer tuning adaptation added. Let $\transp{j,\gamma_j}{}{\theta}{\cdot}$ denote the transition distribution of chain at level $j$ using adaptation $\gamma_j$, starting in state $\theta$. Assume $\forall j, \omega_j \in \bigl[ \underline{\omega}, \overline{\omega} \bigr]$ for some $0<\underline{\omega}<\overline{\omega}$, and, at the inner most layer, there exists a minorization constant $\xi_\capJ > 0$ such that
    $\left\|\transp{\capJ, \gamma_{\capJ}}{M}{\theta}{\cdot} - \psi_{\capJ}(\cdot) \right\| \leq (1 - \xi_{\capJ})^M$. 
    Then, 
     \begin{enumerate}[label=(\alph*)]
     \item Simultaneous uniform ergodicity: For all $\tau > 0$. there  exists $n = n(\tau) \in \mathbb{N}$ such that 
    \begin{equation}
        \left\|\transp{\capJ, \gamma_{\capJ}}{n}{\theta}{\cdot} - \psi_j(\cdot) \right\| \leq \tau
    \end{equation}
    for all $\theta \in \Theta$ and $\gamma_j \in \Gamma_j$.
        \item Diminishing adaptation: 
        The amount of adaptation diminishes in probability with the number of steps $t$ in the adaptation as 
    \begin{equation}
        \lim_{t \rightarrow \infty} \sup_{\theta} \left\| \transp{j, \gamma_j^t}{}{\theta}{\cdot} - \transp{j, \gamma_j^{t+1}}{}{\theta}{\cdot} \right\| = 0 .
    \end{equation}
\end{enumerate}
\end{lemma}

Proof for Lemma \ref{lem:layertuninglem} can be found in Appendix \ref{proof:app-layer-tuning}.

\begin{theorem}
\label{thm:layertuning}
    \flarename\ with an adaptive layer tuning parameter is ergodic. 
\end{theorem}

\begin{proof}
    We use Lemma \ref{lem:layertuninglem} to show the conditions necessary in Theorem 1 of Roberts and Rosenthal (2007) \cite{diminishing-adap-ergodicity}. This shows that the adaptive algorithm is ergodic. 
\end{proof}

\section{Experiments}

We measure the efficiency of the MCMC methods tested using the effective
sample size (ESS) \cite{Ripley87} estimated across multiple chains as 
\begin{equation}
   N_{ESS} = \sfrac{\left(N \cdot K\right)}{\left(1 + 2\sum_{k=1}^{2m+1}\rho(k)\right)} 
\end{equation}
where $N$ is the number of samples, $K$ is the number of chains,
$\rho(k)$ is the lag-$k$ correlation, 
and $m$ is the largest value such that $\rho(2m) + \rho(2m+1) > 0$. We compute ESS for each parameter 
for the ``bulk'' (entire distribution) and ``tail'' (largest and smallest $5\%$ of the samples) of the distributions.

We compare \flarename\ with standard Metropolis-Hastings (with proposal adaptation introduced by Haario et al. \cite{bj/1080222083}) and other multi-fidelity MCMC methods: Multilevel
Delayed Acceptance MCMC (MLDA) \cite{doi:10.1137/22M1476770}, MLDA with Adaptive Error Model
(AEM) \cite{AEM, doi:10.1137/22M1476770}, Multi Level MCMC (MLMCMC) \cite{MLMCMC,
MLDA}, and Multi-fidelity Monte Carlo \cite{NEURIPS2022_8803b9ae}. We give more detail about the different methods used for comparison: 

\begin{enumerate}[itemsep=0pt]
    \item MLDA without any adaptation: Introduced by Lykkegaard et al. \cite{doi:10.1137/22M1476770}, this method uses recursive chains of approximations as proposals. However, there is no adaptation being done to ``correct" the approximations. Even the authors note that without adaptation, the chains do not mix well, and have poor effective sample sizes. 
    \item MLDA with AEM: Extended by Lykkegaard et al. \cite{doi:10.1137/22M1476770} in the same paper, this method uses a similar structure to the above. They also use Adaptive Error Model (AEM) as a way to deal with the discrepancies between the different layers which uses a telescoping sum of differences in the mean of the approximations. They demonstrate with the subsurface flow model that MLDA with AEM leads to good mixing and high ESS. We demonstrate similar results for two of our experiments. Our subsurface flow model experiment uses the same fidelities as set up by the original authors; however, we run it to collect more samples using a higher number of chains. 
    \item Multi Level MCMC (MLMCMC): This method was proposed by Dodwell et al. \cite{MLMCMC} and was then applied to MLDA. A quantity of interest, $Q$, is proposed that is
    related to the parameters of the model. The samples drawn from the posterior are used to reduce the variance of $Q$. Since in MLDA, samples are not only drawn
    from a ``true'' posterior, but also approximations, the samples from the approximate levels are used to reduce the variance of $Q$.  They state that it thus
    requires fewer samples to achieve the same variance. Using a telescopic sum, the difference of $Q$ estimates between levels are used to correct $Q$ with respect
    to the next coarser level. For the pendulum model, $Q$ is the mean of the outputs. For the subsurface flow experiment used by the original authors of MLDA, $Q$ is the hydraulic head at some fixed point $(x, y)=(0.5, 0.45)$; that is, the model PDE is solved at these points at each level using samples from the coarser approximate level. 
    \item Multi-Fidelity Monte Carlo (MFMC): This method was proposed by Cai and Adams \cite{NEURIPS2022_8803b9ae}. It uses a continuum of models with increasing fidelity and has a
    single Markov chain with a random choice of the fidelity, $K$, at each step. The fidelity $K$ is part of the sampled state-space and therefore also part of the proposal distribution and acceptance probability. 
    We map $K$ to a reasonable range of fidelities for each experiment. For the pendulum model, we map $K$ to the error tolerance of the integrator, $\epsilon$,
    as $\epsilon = e^{\sfrac{K}{10}} + 10^{-6}$. 
    For the subsurface flow experiment, we let the grid resolution be equal to $10K$ ($K$ is the sampled fidelity of this method) in order to map the resolution to the fidelity range expected by the algorithm's implementation. The samples from this method are not from the true posterior, but rather can be corrected to estimate an expectation (like the mean).  Therefore, we do not plot the evolution of effective samples with respect to time in the results, as they are not samples from the true posterior.
\end{enumerate}

\newrobustcmd\BB{\DeclareFontSeriesDefault[rm]{bf}{b}\bfseries}

\begin{table}[t]
 \caption{Mean ESS for bulk and tail distributions across 50 runs for 10 chains each, and mean value of each parameter with std error 
    across all 500 total runs are listed. Average acceptance rates are listed per layer and the total samples are based on the average cost of likelihood evaluations per sample per method. Each chain is run for a total of 1000 seconds.}
    \subfloat[Simple Pendulum]{
    \label{tab:pend-append}
    \resizebox{0.96\textwidth}{!}{%
    \begin{tabular}{ccccccccccllc}
    \hline
    \multirow{2}{*}{} & \multicolumn{4}{c}{$\alpha$} & \multicolumn{4}{c}{$L$} & \multicolumn{3}{c}{\begin{tabular}[c]{@{}c@{}}acceptance\\  rate\end{tabular}} & \multirow{2}{*}{\begin{tabular}[c]{@{}c@{}}total \\ samples\end{tabular}} \\ \cmidrule(lr){2-5}\cmidrule(lr){6-9}\cmidrule(lr){10-12}
     & \begin{tabular}[c]{@{}c@{}}bulk \\ ESS/s\end{tabular} & \begin{tabular}[c]{@{}c@{}}tail \\ ESS/s\end{tabular} & mean & sd & \begin{tabular}[c]{@{}c@{}}bulk \\ ESS/s\end{tabular} & \begin{tabular}[c]{@{}c@{}}tail \\ ESS/s\end{tabular} & mean & sd & $j = 0$ & $j = 1$ & $j = 2$ &  \\ \midrule
    MCMC & 21.43 & 29.94 & 0.953 & 0.492 & 17.94 & 29.94 &  1.265 & 0.355 & 0.29 &  &  & 100000 \\ 
    [0.5em]
    \flarename (single) & \BB 52.47 & \BB 58.42 &  1.081 & 0.019 &  \BB 45.85 &  \BB 54.95 & 1.372 & 0.016 & 0.98 & 0.24 &  & 60000 \\ 
    AEM MLDA (single) & 39.41 & 53.91 & 1.064 & 0.062 & 42.20 & 52.24 & 1.371 & 0.051 & 0.98 & 0.27 &  & 50000 \\
    MLMCMC (single) & 25.03 & 34.43 & 1.044 & 0.291 & 21.54 & 32.99 & 1.351 & 0.167 & 0.92 & 0.32 &  & 44500 \\
    MLDA (single) & 11.64 & 1.49 & 1.059 & 0.041 & 15.36 & 7.73 & 1.361 & 0.022 & 0.91 & 0.31 &  & 45000 \\
    [0.5em]
    \flarename (double) & \BB 64.26 & \BB 72.10 & 1.086 & 0.001 & \BB 56.68 & \BB 68.20 &  1.374 & 0.009 & 0.99 & 0.86 & 0.29 & 35000 \\
    AEM MLDA (double) & 56.92 & 65.53 & 1.085 & 0.006 & 50.95 & 58.51 & 1.375 & 0.003 & 0.98 & 0.89 & 0.3 & 25000 \\ 
    MLMCMC (double) & 33.82 & 37.05 & 1.049 & 0.015 & 33.09 & 39.94 & 1.361 & 0.015 & 0.90 & 0.81 & 0.28 & 25500 \\
    MLDA (double) & 10.01 & 7.54 & 1.053 & 0.035 & 4.96 & 12.99 & 1.362 & 0.024 & 0.86 & 0.72 & 0.24 & 30000 \\
    [0.5em]
    MFMC & & & 1.075 & 0.057 & & & 1.348 & 0.132 & 0.46 & & & 370000 \\
    \midrule
    \end{tabular}}}
    
    \subfloat[Subsurface Flow model]{
    \label{tab:hydro-append}
    \huge
    \resizebox{0.96\textwidth}{!}{%
    \begin{tabular}{ccccccccccccccllc}
    \hline
    \multirow{2}{*}{} & \multicolumn{4}{c}{$\theta_1$} & \multicolumn{4}{c}{$\theta_2$} & \multicolumn{4}{c}{$\theta_3$} &  \multicolumn{3}{c}{\begin{tabular}[c]{@{}c@{}}acceptance\\  rates\end{tabular}} & \multirow{2}{*}{\begin{tabular}[c]{@{}c@{}}total \\ samples\end{tabular}} \\ \cmidrule(lr){2-5}\cmidrule(lr){6-9}\cmidrule(lr){10-13}\cmidrule(lr){14-16}
      & \begin{tabular}[c]{@{}c@{}}bulk \\ ESS/s\end{tabular} & \begin{tabular}[c]{@{}c@{}}tail \\ ESS/s\end{tabular} & mean & sd & \begin{tabular}[c]{@{}c@{}}bulk \\ ESS/s\end{tabular} & \begin{tabular}[c]{@{}c@{}}tail \\ ESS/s\end{tabular} & mean & sd & \begin{tabular}[c]{@{}c@{}}bulk \\ ESS/s\end{tabular} & \begin{tabular}[c]{@{}c@{}}tail \\ ESS/s\end{tabular} & mean & sd & $j = 0$ & $j = 1$ & $j = 2$ &  \\ \midrule
    
    MCMC & 5.35 & 7.81 & -0.457 & 0.0037 & 5.49 & 8.04 & 0.466 & 0.0036 & 5.63 & 8.19 & 0.076 & 0.0034 & 0.27 &  &  & 10000 \\ [0.5em]
    \flarename\ (single) &  \BB 8.74 &  \BB 11.99 &  -0.460 & 0.0030 &  \BB 8.58 &  \BB 11.81 & 0.467 & 0.0036 & \BB 8.63 & \BB 11.58 & 0.076 & 0.0028 &  0.98 & 0.26 &  & 8365 \\
    AEM MLDA (single) & 7.77 & 4.67 & -0.459 & 0.0032 & 6.96 & 4.29 & 0.466 & 0.0033 & 8.16 & 6.24 &  0.076 & 0.0030 & 0.99 & 0.29 &  & 4100 \\ 
    MLMCMC (single) & 4.35 & 5.86 & -0.460 & 0.0028 & 4.37 & 3.27 & 0.465 & 0.0031 & 5.29 & 6.49 & 0.077 & 0.0021 & 0.93 & 0.32 &  & 8000  \\ 
    MLDA (single)  & 3.65 & 1.52 & -0.463 & 0.0035 & 4.78 & 1.56 & 0.490 & 0.0036 & 3.33 & 3.02 & 0.075 & 0.0031 & 0.87 & 0.34 &  & 8250 \\  [0.5em]
    \flarename\ (double) &  \BB 15.49 &  \BB 21.35 & -0.460 & 0.0026 & \BB 15.141 &  \BB 19.69 &  0.469 & 0.0026 & \BB 15.07 & \BB 19.76 & 0.077 & 0.0023 & 0.99 & 0.95 & 0.3 & 6500 \\
    AEM MLDA (double) & 13.29 & 14.70 & -0.460 & 0.0035 & 12.32 & 13.22 & 0.468 & 0.0035 & 12.70 & 14.45 &  0.077 & 0.0033 & 0.99 & 0.93 & 0.24 & 2285 \\ 
    MLMCMC (double) & 8.46 & 10.49 & -0.459 & 0.0027 & 8.04 & 9.13 & 0.468 & 0.0030 & 7.47 & 5.81 &  0.077 & 0.0024 & 0.92 &0.89 & 0.31 & 6000 \\
    MLDA (double) & 4.48 & 5.05 & -0.461 & 0.0036 & 5.33 & 6.60 & 0.469 & 0.0033 & 5.28 & 7.35 & 0.076  & 0.0030 & 0.93 & 0.91& 0.27 & 6350 \\
    [0.5em]
    MFMC & & & -0.475 & 0.026 & & & 0.418 & 0.041 & & & 0.102 & 0.003 & 0.39 & & & 5700 \\
    \midrule
    \end{tabular}}}
\end{table}

For the MLDA-based methods, we use the authors' implementations in the 
open-source probabilistic programming package PyMC3 \cite{Salvatier2016} by Lykkegaard et al.
\cite{doi:10.1137/22M1476770}
For MFMC, we use the author-provided implementation.

These implementations have significant computational overhead
compared with our implementation of \flarename.  Therefore, we only measure the
time taken in likelihood computation (which is the same code for all methods).

We present three different experimental posterior sampling problems across
different scientific domains: a simple pendulum, a hydrology simulation that
was used by prior methods as a benchmark, and a cosmology simulation that stresses computational limits.

Because of the computational expense of the cosmology simulation, we are
not able to collect a sufficient number of samples to get a reliable estimate
of the effective sample sizes. Instead, we compare our estimates to
those in the cosmology literature.

For each experimental domain, we construct three fidelities by adjusting 
the relevant simulation parameter.  In all cases, we measure our abilities
to sample from the highest fidelity ($j\!\!=\!\!0$).  For methods labeled
``(single),'' there is a single higher fidelity layer ($J=1$).  For methods
labeled ``(double),'' there are two higher fidelity layers ($J=2$): the
one from the ``(single)'' experiments, plus one more that is even more coarse.
For the coarsest fidelity, a Gaussian proposal distribution is used with an adaptive covariance matrix.

In the pendulum and cosmology experiments, this normal distribution is reflected
to keep parameters within their respective ranges. 
\flarename\ can be extended beyond $J\!=\!2$ layers. However, just two layers improves over the standard MCMC and other multi-level methods significantly.
Layers' costs should be roughly orders-of-magnitude different in computational costs.  For these examples, $J\!=\!2$ is the limit of how many layers can
practically be constructed with orders-of-magnitude different computational costs.

\subsection{Simple Pendulum}
\label{ss:pendulum}

The equation of motion for a pendulum of length $L$, mass $M$, and initial angle $\alpha_0$ is
    $\ddot{\alpha}= -\left(\sfrac{g}{L}\right) \sin \alpha$. 
Our goal is to sample from the posterior of the distribution the two
parameters $\theta=(L, \alpha)$ conditioned on the observations of
$\alpha$ at three irregularly spaced times during the motion:
$\alpha(1)=-0.85$, $\alpha(2.3) = 0.9$, $\alpha(5.0)=0.95$.  Observations of these angles
are assumed to be corrupted by Gaussian noise with known standard
deviation: $\sigma=0.1$.
Different
fidelities correspond to adjusting the error tolerance of an adaptive Runge-Kutta 4(5) ODE integrator \cite{DORMAND198019} 
($10^{-3}$ or $10^{-6}$ in our experiments) with stepsize control and dense output \cite{Hairer}.
As a separate coarsest layer of approximation, we use the
small angle approximation (which does not hold for the observations),
$\sin(\alpha) \approx \alpha$, reducing the equation of motion to a simple harmonic motion 
which can be solved analytically as $\alpha(t) =
\alpha_0\cos(t\sqrt{\sfrac{g}{L}})$.
The difference between the finest fidelity posterior 
and this small angle approximation
is shown in Figure~\ref{fig:post}.

\begin{wrapfigure}{r}{\figwidth}
	\vspace{-\baselineskip}
	\resizebox{\linewidth}{!}{\input{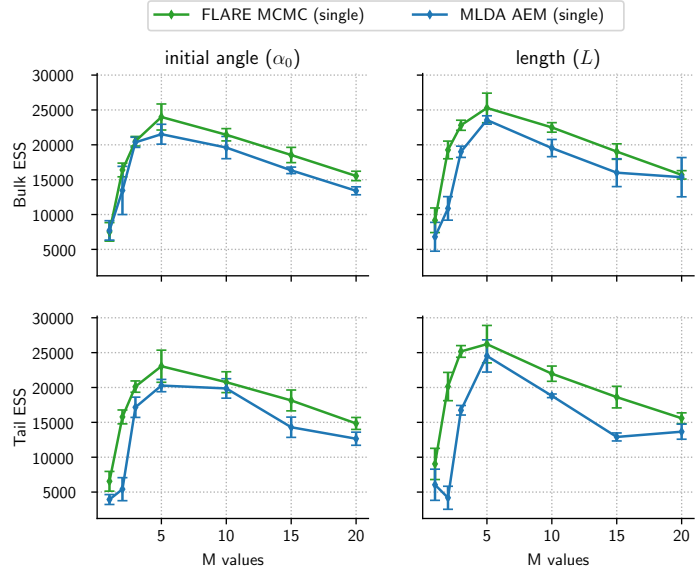}}
\caption{ESS as a function of M for two single layered ($J = 1$) methods with fixed computation time of 500 seconds each ran across 50 different chains.}
	\vspace{-\baselineskip}
\label{fig:ess-comp}
\end{wrapfigure}

\paragraph{Results}
To judge the importance of setting $M$, we evaluated our method across different values of $M$ with fixed total computation time. The effective sample size (ESS) is plotted as a function of $M$ for the two competing multilevel methods in Figure \ref{fig:ess-comp}. We can see that beyond $M=5$, the increased computation time from running longer inner subchains leads to a decrease in overall sampling efficiency, indicating that $M = 5$ offers the best trade-off between computation time and sampling efficiency. Therefore, the layered subchains were run for $M=5$ steps.

We ran 10 chains of the finest fidelity for all methods.We
replicated this experiment (of 10 chains) 50 times. 
Table~\ref{tab:pend-append} summarizes the mean effective sample size per second (ESS/s) and the average mean of the parameters across all 500 chains with the standard deviation along with the acceptance rates for every layer.

Figure~\ref{fig:pend-figs} shows ESS (across all 10 chains) as a function of
computation time for each method, with the total number of samples ($N$)
generated in 1000 seconds.  The standard deviations
are plotted as (barely visible) vertical bars. For the sake of readability, we have separated our plots to show how each method performs with one level of nesting (single) and two levels of nesting (double).

We note that the MFMC method obtains significantly more samples in the same time budget. This result arises from its randomized fidelity selection, which collects
samples at low-fidelity evaluations more frequently than other methods.  Since the samples from this method are not from the true posterior, we do not list the effective samples in the table.
However, for comparison, the average bulk and tail ESS/s for 50 runs of 10 chains each measured for parameters $[\alpha, l] $ are $[24.80, 20.46]$ and $[26.52, 30.61]$
for MFMC. 

All methods are able to improve by using more fidelities.
Our \flarename\ method is consistently and significantly better than the other
methods (including the best one, MLDA with AEM as shown in the table) in terms of ESS/s in both the bulk and tail of the distribution. The acceptance rates for the
multilevel methods indicate that while the coarsest level accepts about a third of the samples (consistent with Gelman et al. \cite{Gelman1996}), the proposed sample from that
level is accepted by the finer levels frequently since it was already accepted by an approximate coarse level. The mean of the parameter across different runs of
the standard MCMC has a higher standard deviation compared to the other multilevel methods suggesting that some runs of MCMC do a poor job at finding the modes in the posterior. \flarename\ produces more samples in the same amount of time compared to its competing method MLDA reflecting that our sampler requires fewer likelihood evaluations per step. 

\newcommand\figscale{0.65}

\begin{figure}
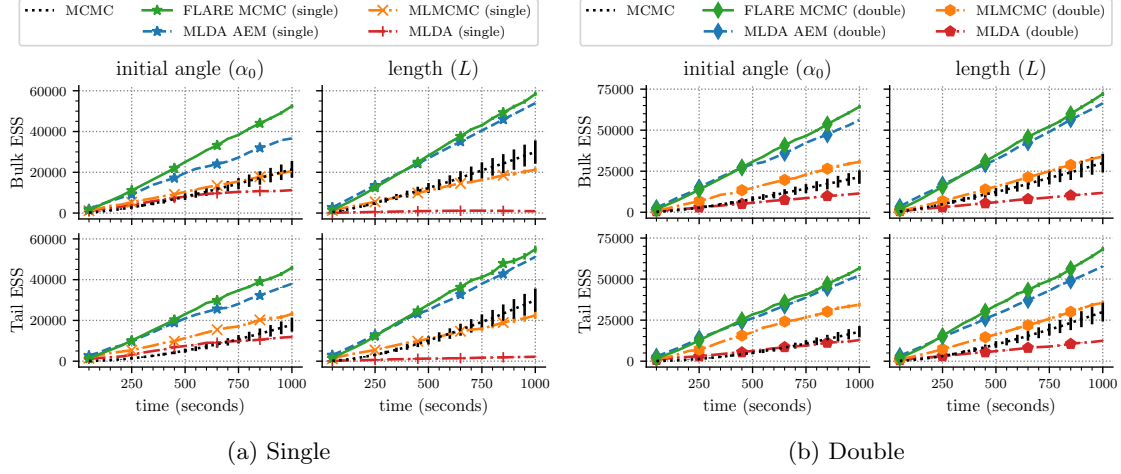

\centering

\subfloat[][Single]{\scalebox{\figscale}{\input{new-plots/pendulum-fixed-single-grid-new.pgf}}\label{fig:pend-single}}
\subfloat[][Double]{\scalebox{\figscale}{\input{new-plots/pendulum-fixed-double-grid-new.pgf}}\label{fig:pend-double}}

\caption{Pendulum Model: ESS for bulk and tail across 50 runs (mean and std.~dev.) for single and double layers of nesting.
}
	\vspace{-\baselineskip}

\label{fig:pend-figs}
\end{figure}

\begin{figure}
    {\centering \resizebox{\linewidth}{!}{\input{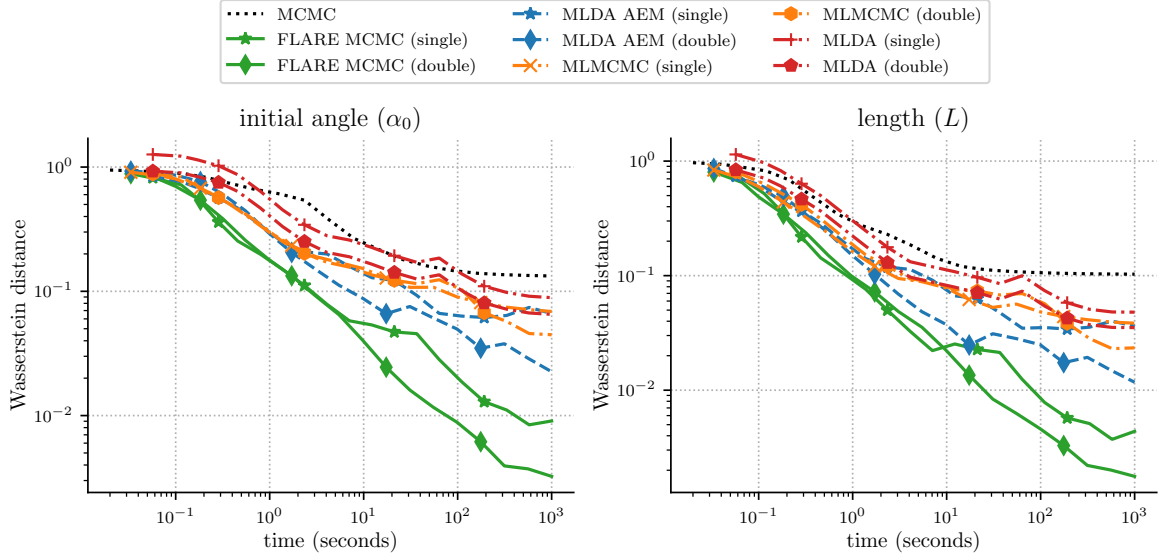}}}
\caption{Wasserstein distance to true distributions for the pendulum model.
}
	\vspace{-\baselineskip}
\label{fig:conv}
\end{figure}

Figure \ref{fig:conv} shows the distributional distance to the true distribution as a function of number of samples, averaged over 500 chains. The true mean and
standard deviation are unknown for all the experiments in the paper, and measuring the distance between a multi-dimensional distribution which can be evaluated
(only up to a normalizing constant) and a distribution represented by samples is non-trivial. However, we have analyzed the $1$-dimensional marginals of the
pendulum model in the following way. We evaluate the true unnormalized distribution on a grid, normalize it, project it to the marginal of interest and then treat
it as (weighted) samples for a sample-to-sample Wasserstein distance between it and the samples from the MCMC methods. As we refine the grid, the distances become
smaller (for almost all methods). We refine the grid until these distances stabilize, resulting in about a 1000-by-1000 grid (1 million llh evaluations). From the
figure, it is clear that for both the parameters, \flarename{} ends up with the smallest Wasserstein distance to the true distribution.

\subsection{Estimation of Soil Permeability in Subsurface Flow}

We consider a simple problem in subsurface flow modeling \cite{MLMCMC}. This model was also used to evaluate the MLDA methods by the original authors \cite{doi:10.1137/22M1476770}, and we did not modify the code used by MLDA (except to increase the number of chains and measure time). 

The classical equations governing (steady state) single–phase subsurface flow consist of Darcy’s law coupled
with an incompressibility condition:
\begin{equation}\label{eqn:darcy}
    w + k\nabla p = g \text{ and } \nabla\cdot w = 0
\end{equation}
subject to suitable boundary conditions. All quantities are fields over $\mathcal{D} = [0,1]^2 \subset \mathbb{R}^2$
for these experiments.
Here $p$ denotes the hydraulic head of the
fluid, $k$ is the permeability tensor, $w$ is filtration velocity (or Darcy
flux) and $g$ is the (known) source term.

\begin{figure}
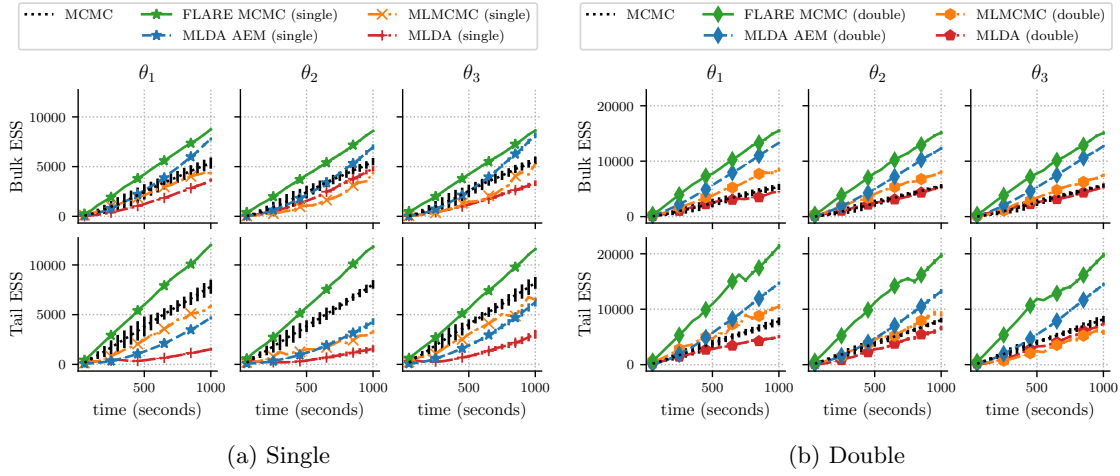

\centering

\subfloat[][Single]{\scalebox{\figscale}{\input{new-plots/hydro-fixed-single-grid-new.pgf}}\label{fig:hydro-append-single}}
\subfloat[][Double]{\scalebox{\figscale}{\input{new-plots/hydro-fixed-double-grid-new.pgf}}\label{fig:hydro-append-double}}

\caption{Subsurface Flow Model: ESS for bulk and tail across 50 runs (mean and std.~dev.) for single and double layers of nesting.
}
	\vspace{-\baselineskip}

\label{fig:hydro-figs}
\end{figure}

We are interested in the permeability given observations 
(with known-variance Gaussian noise) of the
hydraulic head at 16 regularly spaced points in $\mathcal{D}$.
$k$ is simplified to be the gradient of a random scalar field.
The log-Gaussian scalar field is parameterized with a truncated Karhunen-Lo\'{e}ve (KL) expansion
(to three terms, following MLMCMC \cite{MLMCMC}).  These three parameters ($\theta$)
have a standard normal prior and we sample
from their posterior.

Computing the likelihood involves solving a partial differential equation
(PDE) with known boundary conditions for a given $\theta$ and comparing the
results for $p$ at the observation points.  The fidelities correspond to
different grid resolutions for the PDE solver: $120 \times 120$ (highest),
$30\times 30$, and $10\times 10$ (coarsest).

\paragraph{Results}

Previous work reports that $M = 5$ achieves the best trade-off between effective sample size and computation time for this experimental setup \cite{doi:10.1137/22M1476770}. Therefore, we adopt the same value to ensure a fair comparison with our method. Table~\ref{tab:hydro-append} summarizes the same statistics for this model with the same set-up as the pendulum experiments. Figure~\ref{fig:hydro-figs} shows ESS
(across all 10 chains) as a function of computation time for each method. In terms of ESS/s, our method improves over the standard MCMC and outperforms the
multilevel methods for the same amount of likelihood computational budget, especially in the tail of the distribution. All methods converge to similar means of the parameters with low standard deviation among
chains. 

We note that MFMC collects fewer number of samples compared to other methods since each sample requires multiple log likelihood calculations in the same high fidelity to update $K$, leading to significant add up of computational costs.  
For MFMC, the estimated mean ESS/s for
bulk and tail for parameters $[\theta_1, \theta_2, \theta_3]$ are $[0.128,0.081,0.710]$ and $[0.227,0.161,1.107]$.  But, again, the samples from MFMC were never intended to
be interpreted as from the true distribution.

\subsection{Structure Formation in the Universe with N-body Gravitational Simulation}

An important problem in modern-day cosmology is to generate theoretical models of the Universe on very large
scales (tens of $\mathrm{Mpc}$ across) that can be compared to observations. Bayesian inference 
allows cosmologists to measure quantities of fundamental physics significance, such as the
nature of dark energy and dark matter \cite{Peebles:1980}. The theoretical models needed for next generation
telescopes, such as \textsc{euclid} 
\cite{Euclid:2018} and the
Roman Space Telescope (\textsc{wfirst}) \cite{Spergel:2013}, are based on expensive numerical simulations,
some of which require many days of computer time for each evaluation. For such a computationally expensive
model, we show the efficacy of \flarename\ as compared with the standard Metropolis-Hastings algorithm and MLDA. 

\renewcommand{\floatpagefraction}{.8}
\renewcommand\bottomfraction{0.8}

One of the most frequently used summary statistics 
is the
galaxy power spectrum, $\boldsymbol{P}_\mathrm{gg}$ (a bold
$\boldsymbol{P}$ is a power spectrum, not a distribution): 
the two-point clustering of galaxies in Fourier space as a function of the
wavenumber scale, $k$. We use a slightly simplified model for the galaxy
power spectrum for (relative) ease of computation. We perform a forward
simulation which starts from a given set of cosmological parameters and
predicts the galaxy power spectrum. It works by following the
evolution of the Universe under the influence of gravity, from its
beginnings in an almost uniform density state to the diverse collection of
galaxies sitting in dark matter potentials observed today.

We sample from the posterior density of four cosmological parameters: 
\begin{itemize}
	\item[$\theta_1$] The dimensionless Hubble constant, $h$, which
		characterizes the Universal expansion rate and thus the recession
		velocity of distant galaxies. A redshift zero galaxy at distance
		$d$ $\mathrm{Mpc}$ recedes at a speed $v = H_0 d$, where $H_0 = h \times 100 \,\mathrm{km\,s}^{-1}\mathrm{Mpc}^{-1}$. Measuring $h$ is of importance to understand dark energy.
	\item[$\theta_2$] The dimensionless total matter density, $0 < \Omega_0 < 1$. $\Omega_0$ is the energy density of matter as a function of the critical density. $\Omega_0$ is important because it can be used to infer the density of dark matter.
	\item[$\theta_3$] The dimensionless scalar perturbation amplitude,
		$A_s$, of the primordial fluctuations at
		the wavenumber $k = 0.05 \,\mathrm{Mpc}^{-1}$. $A_s$ is of
		interest because it connects to the uncertain high energy physics
		of the Early Universe. Larger values of $A_s$ correspond to a
		clumpier early Universe and so lead to larger $\boldsymbol{P}_\mathrm{gg}$.
	\item[$\theta_4$] The dimensionless linear bias, $b$, which is used to shift the amplitude of our simulated matter power spectrum to match the amplitude of the galaxy power spectrum. This is to account for the difference between observed galaxies and dark matter (which is used by the forward model):
    \begin{equation}
      \boldsymbol{P}_\mathrm{model}(\theta)
      = b^2 \cdot \boldsymbol{P}_\mathrm{dm}(h, \Omega_0, A_s),
    \end{equation}
where $b$ is the scale-independent linear bias and
		$\boldsymbol{P}_\mathrm{dm}$ is the simulated dark-matter power
		spectrum directly computed from the output density field of
		FastPM.
\end{itemize}

The posterior density is conditioned on the galaxy power spectrum from SDSS-III Baryon Oscillation Spectroscopic Survey (BOSS) Data Release 12 (DR12) as our observational data source \cite{BOSS:2013,BOSS:2017}.
We have used a subset of the BOSS data from the North Galactic Cap (NGC) at $z = 0.38$, which includes $\sim
10^6$ galaxies, from Ivanov et al. \cite{Ivanov:2020}.

The likelihood function is a multivariate Gaussian between the galaxy power
spectrum from BOSS, $\boldsymbol{P}_\mathrm{gg}$, and the galaxy power
spectrum from the forward model, $\boldsymbol{P}_\mathrm{model}(\theta)$:
\begin{equation}
   \ln \mathcal{L}(\theta) =
   - \frac{1}{2}
   (\boldsymbol{P}_\mathrm{model}(\theta) - \boldsymbol{P}_\mathrm{gg})^\intercal \boldsymbol{\mathrm{C}}^{-1}
   (\boldsymbol{P}_\mathrm{model}(\theta) - \boldsymbol{P}_\mathrm{gg}) + k.
\end{equation}
$\boldsymbol{\mathrm{C}}$ is the covariance matrix of the galaxy power spectrum, also estimated observationally.

The most expensive part of the forward model, evolution under gravitational force, is computed using FastPM
\cite{Feng:2016}. FastPM has a couple of tunable fidelity parameters.
The size of the region simulated controls the amount of data available and may have a non-linear effect on the accuracy of the result.
We thus fix the size of this region to 1024 $\mathrm{Mpc}/\mathrm{h}$ and instead change the number of particles. More particles in the simulation mean higher resolution, more accurate power spectrum at higher wavenumber $k$, and thus the likelihood function is higher fidelity.
For N-body simulations, the compute time usually scales as $N \log{N}$, where $N$ is the number of particles.
Thus a simulation with a $512^3$ number of particles is $\simeq 80$ times more expensive than a $128^3$ simulation. We therefore set the fidelities by only adjusting the number of particles used in the simulation to be $512$ (highest), $384$, and
$256$ (coarsest).
Note this calculation is distributed across 20 cores (using MPI) and therefore a saving of 1 hour corresponds to 20 core-hours.

\begin{figure}
	\resizebox{\linewidth}{!}{\input{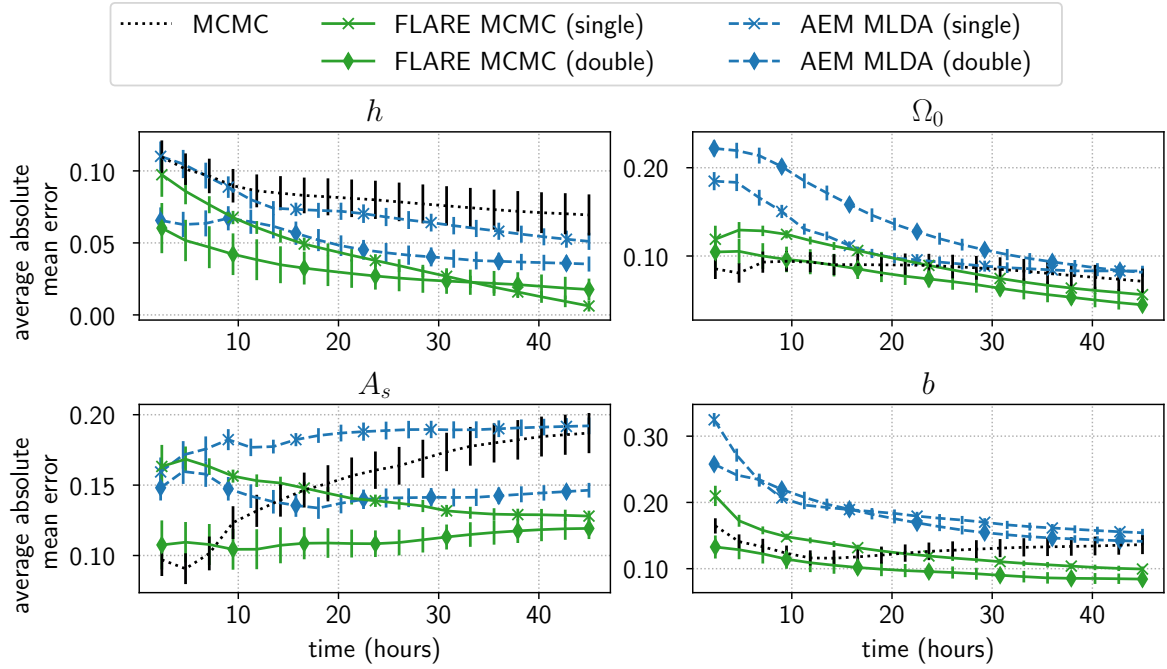}}
\caption{Average (across 10 chains) absolute error in
	mean estimates for the four parameters, as a function of total simulation
	time.  In the allotted 48 hours, the chains each sampled 800 samples for (plain) MCMC, 600 samples for \flarename\ (single), 500 samples for \flarename\ (double), 290 samples for MLDA (single) and 200 samples for MLDA (double)
	}
\label{fig:cosmoplot}
\end{figure}

We compare our estimated distributional means to previous computations
on the same data.  For the parameters $h$ and $\Omega_0$, we compare to
the means reported by Ivanov et al. \cite{Ivanov_2020} on the same data using their
own MCMC simulation ($h = 0.661$ and $\Omega_m = 0.290$).  We have only a single linear bias term, compared
with the multiple such terms of Ivanov et al. \cite{Ivanov_2020}.  Therefore, we can
compare neither $b$ nor $A_s$ (which is heavily related to $b$) to their
results.  Instead, we measure $A_s$ against the best fit value from the Planck
Satellite \cite{plank_2020}, $A_s = 2.09$ and $b = 2$, consistent with comparable BOSS measurements \cite{Ivanov_2020}.
While these are modes (and not means), they
are the best independent estimates we can obtain.

\begin{figure}
	\resizebox{\linewidth}{!}{\input{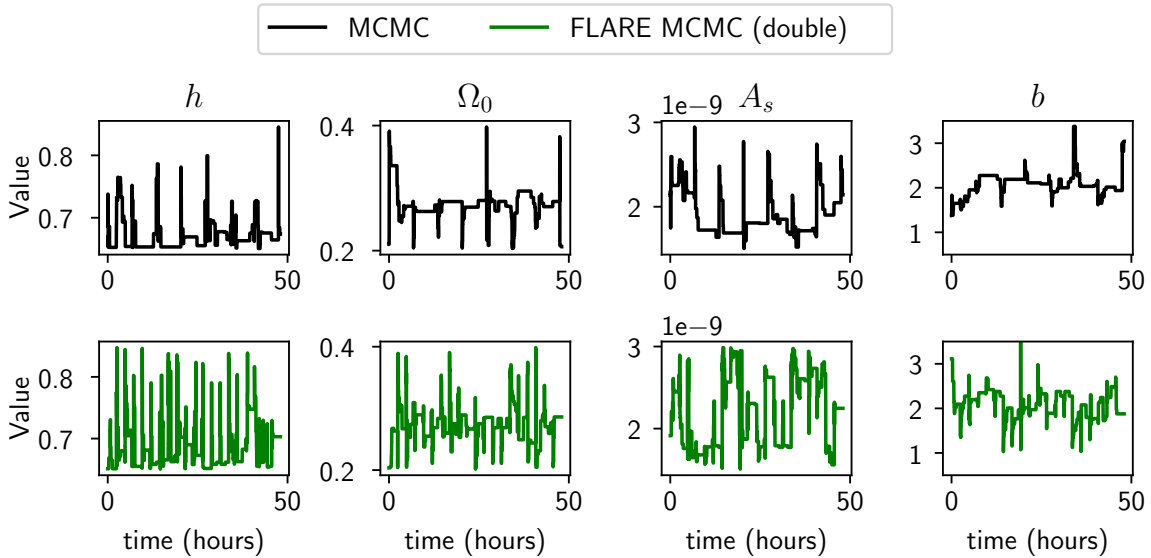}}
\caption{Cosmology model; Trace plot for a random single run of each method}
\label{fig:cosmo-trace}
\end{figure}

\paragraph{Results}

Extreme running time dictated smaller values for $M$ for this experiment. We reduced them by a factor of 2 (approximately) and used $M = 2$ for inner substeps. Figure \ref{fig:cosmoplot} shows that the \flarename\ methods converge to the
mean values from previous literature better than the standard Metropolis
Hastings method using fewer samples and less time.
The $A_s$ parameter has slightly strange behavior.  We can still see better
convergence of our methods.  However, note that the best-fit value of $A_s$
we are taking as ``ground truth'' is measured (with error) from a different dataset,
and thus is likely not the true mean of our posterior. Many large scale structure experiments prefer a lower value of this parameter than Planck, a feature known as the S8 tension \cite{Abdalla_2022}. 
The pairwise plots of
the posterior 
can be found in the Appendix. From the posteriors, it is clear that \flarename\ is
better at approximating the modes of the distribution as compared to MCMC.
Figure \ref{fig:cosmo-trace} shows the trace plot for a random run of MCMC and \flarename. Our method shows better mixing than MCMC and is less likely to reject proposed samples.

\section{Summary}

Many scientific and engineering problems involve simulations or solving differential equations. In this paper,
we present an efficient multi-fidelity layered MCMC that exploits the ability to reduce the accuracy of
models leading to approximations of the posterior. In a recursive, nested fashion, these approximations act as proposals for MCMC-based inference. We add layer
tuning that successfully encourages the approximate proposals to explore the distribution well. We demonstrate with experimental results using models from three
different scientific domains with varying costs that out method, \flarename, is simple, and yet produces more efficient samples than existing adaptive multilevel MCMC methods with the same computational budget.






\appendix

\section{Convergence Rate Proofs}
\label{proof:app-conv-rates}

We show proofs for convergence rates in the main paper here. We first use Lemma \ref{lem:app-mininduction} to show that after $M$ steps of a coarse chain, we can obtain a minorized lower bound that can be recursively used in its finer layer. 


\begin{lemma}
\label{lem:app-mininduction}
    Let $\transp{j}{}{\cdot}{\cdot}$ be the transition distribution of the Markov chain at level $j$ with an invariant target distribution $\fullpi{j}{\cdot}$. For any level $j$, if there exists a $\xi_j > 0$ such that $\transp{j}{1}{\theta_j^0}{\theta_j^1} \geq \xi_j \pi_j(\theta_j^{1})$ for all $\theta_j^0, \theta_j^1 \in \Theta$, then
    \begin{equation}
        \transp{j}{M}{\theta^0_j}{\theta^M_j} \geq \left(1 - (1 - \xi_j)^M\right) \fullpi{j}{\theta^M_j}
        \quad\quad \forall \theta^0_j, \theta^M_j \in \Theta
        \label{eqn: min cond}
    \end{equation}
\end{lemma}

\newcommand\resid[3]{r_{#1}(#2\mid #3)}

\begin{proof}
    We prove this using induction. We can verify the base case for $k = 1$ such that $\transp{j}{1}{\theta^0_{j}}{\theta^1_{j}} \geq (1 - (1 - \xi_{j})^1) \fullpi{j}{\theta^{k}_{j}} = \xi_{j} \fullpi{j}{\theta^1_{j}}$. This is held by the assumption made in the lemma. 
    
    Assume using the induction hypothesis that, $\transp{j}{k}{\theta^0_{j}}{\theta^k_{j}} \geq (1 - (1 - \xi_{j})^k) \fullpi{j}{\theta^{k}_{j}}$. We need to show that $\transp{j}{k+1}{\theta^0_{j}}{\theta^{k+1}_{j}} \geq (1 - (1 - \xi_{j})^{k+1})\fullpi{j}{\theta^{k+1}_{j}}$. \\ 

    Note $\transp{j}{k}{\theta^0_{j}}{\theta^k_{j}}$ can be written as $\transp{j}{k}{\theta^0_{j}}{\theta^k_{j}} = \phi^k \fullpi{j}{\theta^k_{j}} + (1 - \phi^k) \resid{j}{\theta^k_{j}}{\theta^0_{j}}$,
    where $\phi^k$ is the probability that the chain couples to the stationary distribution in $k$ steps, and $\resid{j}{\theta^k_{j}}{\theta^0_{j}}$ is the
    remaining distribution that depends on $\theta^0_{j}$.
    \begin{align*}
        \transp{j}{k+1}{\theta^0_{j}}{\theta^{k+1}_{j}} &= \int \left[\transp{j}{k}{\theta^0_{j}}{\theta^k_{j}} \cdot \transp{j}{1}{\theta^k_{j}}{\theta^{k+1}_{j}} \right] d \theta^k_{j} \\
        &= \int \transp{j}{1}{\theta^{k}_{j}}{\theta^{k+1}_{j}} \left[ \phi^k \fullpi{j}{\theta^k_{j}} + (1- \phi^k) \resid{j}{\theta ^k_{j}}{\theta ^0_{j}} \right] d\theta^k_{j} \\ 
        &= \int \phi^k \fullpi{j}{\theta^k_{j}} \transp{j}{1}{\theta^{k}_{j}}{\theta^{k+1}_{j}}  d \theta^k_{j} + \int (1- \phi^k) \resid{j}{\theta ^k_{j}}{\theta ^0_{j}} \transp{j}{1}{\theta^{k}_{j}}{\theta^{k+1}_{j}} d \theta^k_{j} \\
	   \intertext{We know $\phi^k \geq 1 - (1 - \xi_{j})^k$. Thus,}
        &\geq (1 - (1 - \xi_{j})^k) \fullpi{j}{\theta^{k+1}_{j}} + (1 - \xi_{j})^k \int \resid{j}{\theta ^k_{j}}{ \theta ^0_{j}} \transp{j}{1}{\theta^{k}_{j}}{\theta^{k+1}_{j}} d \theta^k_{j}  \\
	   \intertext{Replacing $\transp{j}{1}{\theta^{k}_{j}}{\theta^{k+1}_{j}}$ between two consecutive samples with the base assumption, }
        & \geq (1 - (1 - \xi_{j})^k) \fullpi{j}{\theta^{k+1}_{j}} + \xi_{j} (1 - \xi_{j})^k \int \resid{j}{\theta ^k_{j}}{\theta ^0_{j}} \fullpi{j}{\theta^{k+1}_{j}} d \theta^k_{j} \\
        &= ( 1 - (1 - \xi_{j})^k )\fullpi{j}{\theta^{k+1}_{j}} + \xi_{j} (1 - \xi_{j})^k \fullpi{j}{\theta^{k+1}_{j}} \\
        &= (1 - (1 - \xi_{j})^{k+1}) \fullpi{j}{\theta^{k+1}_{j}}
        \intertext{Therefore using proof by induction, we have that $ \transp{j}{M}{\theta^0_{j}}{\theta^M_{j}} \geq (1 - (1 - \xi_{j})^M) \fullpi{j}{\theta^{M}_{j}}$.}
    \end{align*}

\end{proof} 
\begin{proof}[Proof of Lemma \ref{lem:minlowerbound}]
The transition kernel is given by
\begin{align*}
    \transp{j}{1}{\theta^{i}_j}{\theta^{i+1}_j} &= \mathcal{A}_j(\theta^{i}_j \rightarrow \theta^{i+1}_j) \cdot \fullprop{j}{\theta^{i+1}_j}{\theta^i_j}
    	+ \delta(\theta^{i+1}_j - \theta^i_j) \int  \left(1-\mathcal{A}_j(\theta^{i}_j \rightarrow \theta'_j)\right) \fullprop{j}{\theta'_j} {\theta^i_j}\, d\theta'_j\\
    &\geq \mathcal{A}_j(\theta^{i}_j \rightarrow \theta^{i+1}_j) \cdot \fullprop{j}{\theta^{i+1}_j}{\theta^i_j} \\
    &= \mathcal{A}_j(\theta^{i}_j \rightarrow \theta^{i+1}_j) \cdot \transp{j+1}{M}{\theta^{0}_{j+1}}{\theta^M_{j+1}} \\
    \intertext{The sample, $\theta_j^{i+1}$ is proposed using the Mth sample from the $j+1$ chain, therefore is the same as $\theta^M_{j+1}$. Thus, from Lemma \ref{lem:app-mininduction}, } 
	&\geq \mathcal{A}_j(\theta^{i}_j \rightarrow \theta^{i+1}_j) \cdot (1 - (1 - \xi_{j+1})^M) \cdot \fullpi{j+1}{\theta^{i+1}_j}\\
    &= \min \left(1, \frac{\fullpi{j}{\theta^{i+1}_j}}{\fullpi{j}{\theta^{i}_j}} \cdot \frac{\fullpi{j+1}{\theta^{i}_j}}{\fullpi{j+1}{\theta^{i+1}_j}}\right) \cdot (1 - (1 - \xi_{j+1})^M )\cdot \fullpi{j+1}{\theta^{i+1}_j}\\
    \intertext{Let $r(\theta) = \frac{\fullpi{j+1}{\theta}}{\fullpi{j}{\theta}}$. Then,} 
	&= \min \left(1, \frac{r(\theta^{i}_j)}{r(\theta^{i+1}_j)} \right) \cdot (1 - (1 - \xi_{j+1})^M) \cdot r(\theta^{i+1}_j) \cdot \fullpi{j}{\theta^{i+1}_j} \\
    &= (1 - (1 - \xi_{j+1})^M) \cdot \min \left( r(\theta^{i+1}_j), r(\theta^i_j)\right) \cdot \fullpi{j}{\theta^{i+1}_j} \\
    & \geq (1 - (1 - \xi_{j+1})^M) \cdot \min_\theta \left( r(\theta) \right) \cdot \fullpi{j}{\theta^{i+1}_j}\\
    & = \xi_j \cdot \fullpi{j}{\theta^{i+1}_j}
    \intertext{ where $\xi_j = (1 - (1 - \xi_{j+1})^M) \cdot \min\limits_{\theta} \left( \frac{\fullpi{j+1}{\theta}}{\fullpi{j}{\theta}} \right)$.   }
\end{align*}
\end{proof}

\section{Optimal M proof}
\label{proof:app-optM}

\begin{proof}[Proof of Lemma \ref{lem:optimal_M}]

Since the ratio $\min\limits_{\theta} \left( \frac{\pi_{j+1}(\theta)}{\pi_j(\theta)} \right)$ is constant with respect to $M_j$, we simplify the objective function to be maximized as: 
\begin{equation*}
    f(M_j) = \frac{1-(1-\xi_{j+1})^{M_j}}{b_j+M_jB_{j+1}}, \qquad M_j\ge0,\quad \xi_{j+1}\in(0,1),\quad b_j,B_{j+1}>0.
\end{equation*}

We define
\begin{equation*}
    c:=1-\xi_{j+1}\in(0,1),\qquad \beta:=\frac{b_j}{B_{j+1}},\qquad \Upsilon:=-\log c>0 \qquad\text{and}\qquad \mu := \beta+\frac{1}{\Upsilon}.
\end{equation*}

Then, 
\begin{equation*}
    f(M_j)=\frac{1-c^{M_j}}{\beta B_{j+1}+M_{j}B_{j+1}}=\frac{1-c^{M_j}}{B_{j+1}(\beta+M_{j})}.
\end{equation*}

Differentiating the objective function, 
\begin{equation*}
    f'(M_j)=\frac{-c^{M_j}\log c\,(\beta+M_j)- (1-c^{M_j})}{(\beta+M_j)^2*B_{j+1}}.
\end{equation*}

Setting the derivative to zero and using $\Upsilon = - \log c$, 
\begin{equation*}
    -c^{M_j}\log c\,(\beta+M_j)=1-c^{M_j}
\quad\Longrightarrow\quad
c^{M_j}\big(1+\Upsilon(\beta+M_j)\big)=1.
\end{equation*}

Since $c^{M_j} = e^{\log(c^{M_j})} = e^{M_j\log c} = e^{M_j(-\Upsilon)}$, we replace $c^{M_j}=e^{-\Upsilon M_j}$ to obtain
\begin{equation*}
    \big(1+\Upsilon(\beta+M_j)\big)e^{-\Upsilon M_j}=1.
\end{equation*}

By the definition of $\mu$, we have $1+\Upsilon(\beta+M_j)=\Upsilon\mu+\Upsilon M_j$. Therefore, 
\begin{equation*}
    (\Upsilon\mu+\Upsilon M_j)e^{-\Upsilon M_j}=1.
\end{equation*}

Let $y = \Upsilon \mu + \Upsilon M_j$. Then, $\Upsilon M_j = y - \Upsilon \mu$, substituting this gives 

\begin{align*}
    ye^{-(y-\Upsilon\mu)} &= 1 \\
    ye^{-y} &= e^{-\Upsilon\mu} \\
    - ye^{-y} &= -e^{-\Upsilon\mu}.
\end{align*}

Using the Lambert W function for branch $k = -1$,
\begin{equation}
        y = -W_{-1}\big(-e^{-\Upsilon\mu}\big).
\end{equation}

Since  $\Upsilon M_j = y - \Upsilon\mu$, we get
\begin{equation*}
    M_j = \frac{y}{\Upsilon}-\mu = -\frac{1}{\Upsilon}W_{-1}\big(-e^{-\Upsilon\mu}\big)-\mu.
\end{equation*}

Therefore, the maximizer for our objective function is 
    \begin{equation}
            M_j^* = -\frac{1}{\Upsilon}W_{-1}\big(-e^{-\Upsilon\mu}\big) - \mu
    \end{equation}

    where $\Upsilon = -\log(1-\xi_{j+1})$, $\mu = \frac{b_j}{B_{j+1}} + \frac{1}{\Upsilon}$ and $W_{-1}$ is $-1$ branch of the Lambert $W$ function.

\end{proof}

\section{Layer Tuning Ergodicity Proof}
\label{proof:app-layer-tuning}

\begin{proof}[Proof of Lemma \ref{lem:layertuninglem}]
\begin{enumerate}[label=(\alph*)]
    \item 
    Consider layer $J-1$. Using Theorem \ref{thm: conv-thm}, 
    \begin{align*}
\left\| \transp{\capJ-1, \gamma_{\capJ-1}}{M}{\theta}{\cdot} - \psi_{\capJ-1}(\cdot) \right\|
&\leq (1 - \xi_{\capJ-1})^M, \quad
\xi_{\capJ-1} = (1 - (1 - \xi_\capJ)^M)
\min_\theta \frac{\psi_\capJ(\theta)}{\psi_{\capJ-1}(\theta)} \\[2mm]
\intertext{From the definition of layer tuning,}
&= (1 - (1 - \xi_\capJ)^M)
\min_\theta 
\frac{ (\fullpimod{\capJ}{\theta} + \omega_\capJ) 
       \cdot \zetap{\capJ-1}{\omega_{\capJ-1}} }
     { (\fullpimod{\capJ-1}{\theta} + \omega_{\capJ-1}) 
       \cdot \zetap{\capJ}{\omega_{\capJ}} } \\[1mm]
&= (1 - (1 - \xi_\capJ)^M)
\min_\theta 
\frac{ \fullpimod{\capJ}{\theta} + \omega_\capJ }
     { \fullpimod{\capJ-1}{\theta} + \omega_{\capJ-1} } \\
&\quad \times 
\frac{ Z + \omega_{J-1}\cdot V }{ Z + \omega_J \cdot V } \\[1mm]
\shortintertext{where $\zeta$ is the normalizing constant of the new distribution  that depends on $\omega$, $Z$ is the
		normalizing constant of the original distribution, and $V$ is the volume of $\Theta$. With the bounds for $\omega$ from the assumption,}
&\geq (1 - (1 - {\xi_\capJ})^M)\min_\theta \left( \frac{\fullpimod{\capJ}{\theta} + \underline{\omega}}{\fullpimod{\capJ-1}{\theta} +
		\overline{\omega}} \right) \left (\frac{Z + \underline{\omega}\cdot V}{Z + \overline{\omega} \cdot V} \right) \\
&\triangleq \overline{\xi}_{\capJ-1}
\end{align*}

    By induction with base case at layer $J$, 
    \begin{equation}
    \label{eqn:ergodicityadapt}
        \forall j, \left\|\transp{j, \gamma_j}{n}{\theta}{\cdot} - \psi_j(\cdot) \right\| \leq (1 - \xi_j)^n \text{ where } \xi_j \geq  \overline{\xi}_{j}\,\,.
    \end{equation}

      We need to show that for all $\tau > 0$. there  exists $n = n(\tau) \in \mathbb{N}$ such that 
    \begin{equation*}
       \left\|\transp{j, \gamma_j}{n}{\theta}{\cdot} - \psi_j(\cdot) \right\| \leq \tau
    \end{equation*} for all $\theta \in \mathcal{X}_j$ and $\gamma_j \in \Gamma_j$.

    From Equation \ref{eqn:ergodicityadapt}, we want
    \begin{align*}
        (1 - \xi_j)^n & \leq \tau \\
        n & \geq \frac{\ln \tau}{\ln (1 - \xi_j)} 
        \intertext{Since $\ln (1 - \xi_j)  \leq \ln (1 - \overline{\xi_j})$, }
         n & \geq \frac{\ln \tau}{\ln (1 -\overline{\xi_j})} \,\,.
    \end{align*}
    Thus, for all $\tau>0$, there exists $n = \max\limits_\tau \frac{\ln \tau}{\ln (1 -\overline{\xi_j})}$ such that
        $\left\|\transp{\capJ, \gamma_{\capJ}}{n}{\theta}{\cdot} - \psi_j(\cdot) \right\| \leq \tau$.
    \item     At every step $t$, the change in $\gamma_j$ maps to change in $\smooth{j+1}$. Diminishing adaptation is guaranteed by a gradient descent algorithm
    with diminishing stepsize that updates $\omega_{j+1}$ at each layer to minimize the Kullback-Leibler divergence between layers $\psi_j$ and $\psi_{j+1}$. At each step of the GD algorithm, $\smooth{j+1}$ is updated as $\smooth{j+1}^{t+1} = \smooth{j+1}^{t} - \eta_i \smoothder{j+1}\obj{j+1}$. To get diminishing adaptation, the update needs to converge as
    \begin{align*}
        \lim_{t \rightarrow \infty} \left\|\eta_t \smoothder{j+1}\obj{j+1}\right\| \approx 0. 
    \end{align*}

    Since we bound $\gamma_j \leftrightarrow \omega_{j+1}$ away from zero, 
    \begin{align*}
        \smoothder{j+1}\obj{j+1} &= \frac{1}{\fullpimod{j+1}{\theta_{j+1}^0} + \smooth{j+1}} - \frac{1}{\fullpimod{j+1}{\theta_{j+1}^M} + \smooth{j+1}} \\
        & \leq \frac{1}{\fullpimod{j+1}{\theta_{j+1}^0} + \smooth{j+1}} \\ 
        & \leq \frac{1}{\fullpimod{j+1}{\theta_{j+1}^0}  + \underline{\omega}} \\ 
        & \leq \frac{1}{\underline{\omega}} \\ \intertext{Therefore, the update is}
        \lim_{t \rightarrow \infty} \left\|\eta_t \smoothder{j+1}\obj{j+1}\right\| &\leq \lim_{t \rightarrow \infty} \left\|\eta_t
	   \frac{1}{\underline{\omega}}\right\| \\
        &\leq \frac{1}{\underline{\omega}} \lim_{t \rightarrow \infty} \eta_t
        \intertext{If the stepsize, $\eta_t$ asymptotes to $0$, adapation decreases to $0$ as $t\rightarrow\infty$.  Therefore 
    $\lim\limits_{t\rightarrow\infty} \left\|\smooth{j+1}^{t} - \smooth{j+1}^{t+1} \right\| = 0$, and
    thus
        $\lim\limits_{t \rightarrow \infty} \sup\limits_{\theta} \left\| \transp{j, \gamma_j^t}{}{\theta}{\cdot} - \transp{j, \gamma_j^{t+1}}{}{\theta}{\cdot} \right\| = 0$.
}
    \end{align*}

\end{enumerate}
\end{proof}

\section{Uniform Smoothing Parameter}

Plotted in Figure \ref{fig:smooth} is the evolution of our tuning parameter $\smooth{j}$ as a function of samples collected for one example chain of the doubly nested method for the pendulum model. Each sample at $j = 1$ starts a chain of length $M = 5$ at inner layer $j=2$. The inner layer $j = 2$ uses the small angle approximation of the pendulum as the fidelity. Since it is a poor approximation of the posterior as shown in Fig \ref{fig:post}, we start with a relatively high value of $\smooth{}$. This helps the coarsest layer better explore the high-probability regions. As shown, the tuning parameter $\smooth{}$ converges close to zero after a few samples in both the layers. We use a learning rate of $10^{-3}$ for both the layers. 
\begin{figure}[h!]
	{\centering
	\scalebox{1}{\input{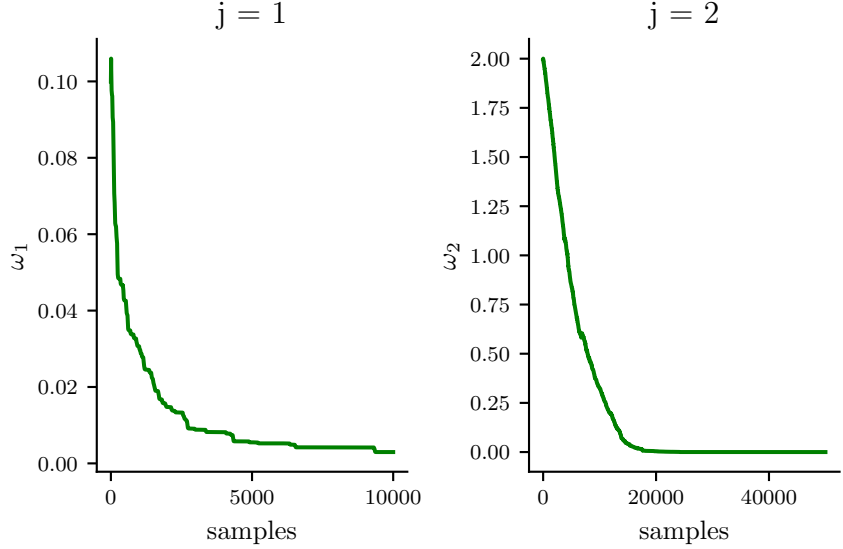}}

	}
\caption{Evolution of $\smooth{j}$ for doubly nested layers $j = 1$ and $j = 2$}
\label{fig:smooth}
\end{figure}

\section{Computational Infrastructure}
Our experiments were performed on a machine with 4 Intel\textregistered\
Xeon\textregistered\ Silver 4214
CPUs running at 2.20GHz for our experiments (a total of 48 cores).  The
machine has 250GB of memory, but memory was never a restriction during our
experiments.

All methods use multiproccessing, that is, each chain is run in parallel
using a different core. The time listed is across one run of a single chain;
however, the effective sample size is calculated across 10 different
chains.

For the pendulum and hydrology models, likelihood calculations were carried
out on a single core.  For the cosmology model, the likelihood calculations
were carried out in parallel across 20 cores.  Therefore, for the cosmology
experiments, saving a day's worth of computation time on the graphs
corresponds to saving 20 days worth of core-hours.

\section{Pairwise Plot for Cosmology Model}

Plotted in Figure \ref{fig:pairwise-mcmc} is the pairwise plot for all four parameters of the cosmology model. MCMC generates more samples in the same period of time, yet these samples have not yet converged to the distribution and are still scattered across the space, compared with the relatively compact \flarename{} samples. 

\begin{figure}[p]
\includegraphics[width=\columnwidth]{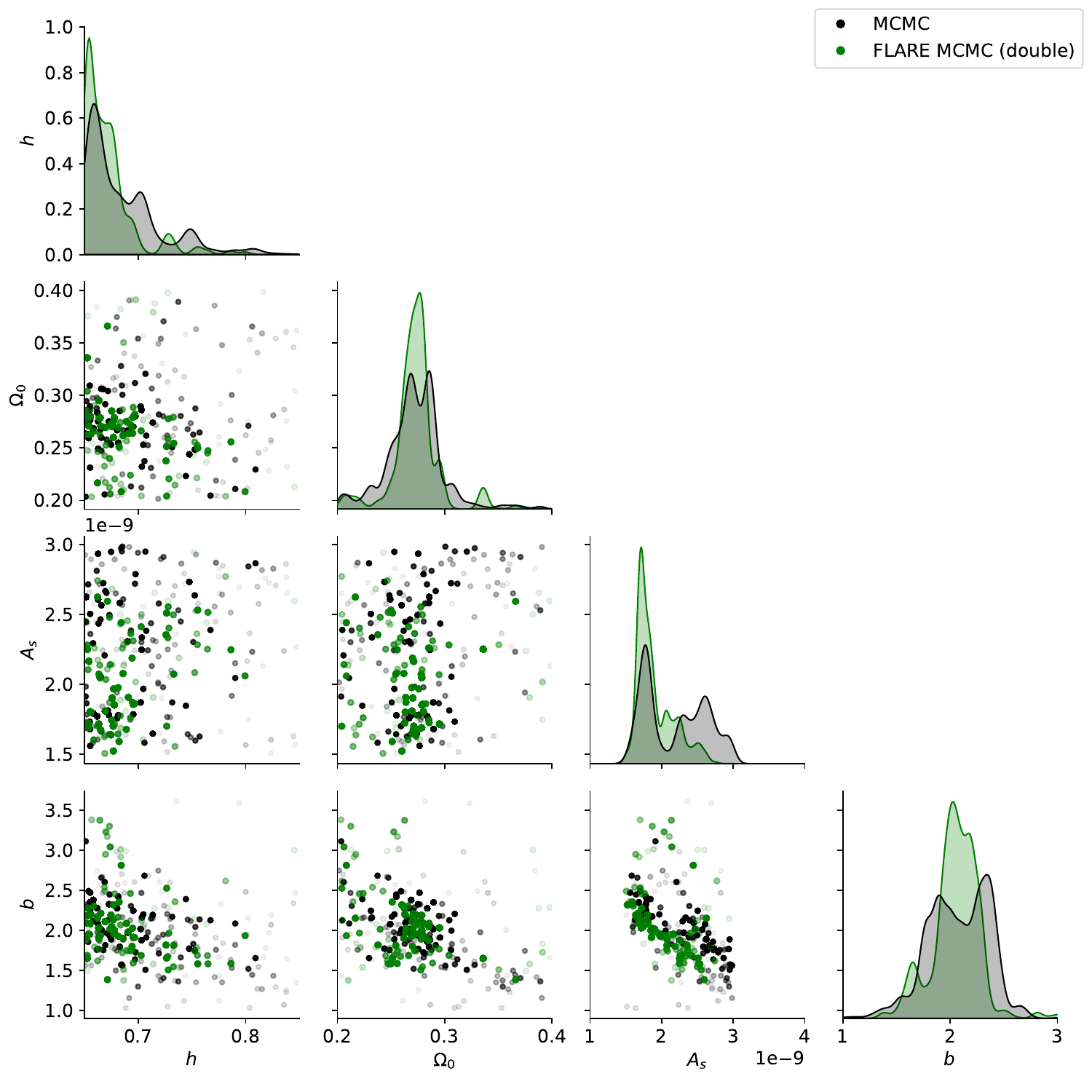}
\caption{Cosmology model: Pairwise plots for all four parameters. Plotted for a random run of samples collected for 48 hours of MCMC (black) and \flarename{} (green).}
\label{fig:pairwise-mcmc}
\end{figure}

\clearpage
\bibliography{references}

@preamble{"\newcommand\mnras{Monthly Notices of the Royal Astronomical Society}"}

@preamble{"\newcommand\aj{Astronomical Journal}"}

@preamble{"\newcommand\apj{Astrophysical Journal}"}

@preamble{"\newcommand\jcap{Journal of Cosmology and Astroparticle Physics}"}

@Article{DORMAND198019,
title = {A family of embedded {R}unge-{K}utta formulae},
journal = {Journal of Computational and Applied Mathematics},
volume = {6},
number = {1},
pages = {19-26},
year = {1980},
issn = {0377-0427},
doi = {https://doi.org/10.1016/0771-050X(80)90013-3},
url = {https://www.sciencedirect.com/science/article/pii/0771050X80900133},
author = {J.R. Dormand and P.J. Prince}
}

@article{DAMCMC,
  author    = {J. Andr{\'e}s Christen and Colin Fox},
  title     = {Markov chain Monte Carlo Using an Approximation},
  journal   = {Journal of Computational and Graphical Statistics},
  volume    = {14},
  number    = {4},
  pages     = {795--810},
  year      = {2005},
  publisher = {Taylor \& Francis},
  doi       = {10.1198/106186005X76983},
  url       = {https://doi.org/10.1198/106186005X76983},
  eprint    = {https://doi.org/10.1198/106186005X76983}
}

@article{MHMCMC1,
 ISSN = {00063444},
 URL = {http://www.jstor.org/stable/2334940},
 author = {W. K. Hastings},
 journal = {Biometrika},
 number = {1},
 pages = {97--109},
 publisher = {[Oxford University Press, Biometrika Trust]},
 title = {{M}onte {C}arlo Sampling Methods Using {M}arkov Chains and Their Applications},
 urldate = {2022-08-08},
 volume = {57},
 year = {1970}
}

@misc{MLDA,
  doi = {10.48550/ARXIV.2012.05668},
  
  url = {https://arxiv.org/abs/2012.05668},
  
  author = {Lykkegaard, Mikkel B. and Mingas, Grigorios and Scheichl, Robert and Fox, Colin and Dodwell, Tim J.},
  
  title = {Multilevel Delayed Acceptance {MCMC} with an Adaptive Error Model in {PyMC3}},
  
  publisher = {arXiv},
  
  year = {2020},
  
  copyright = {Creative Commons Attribution 4.0 International}
}

@book{Ripley87,
  address = {New York, NY, USA},
  author = {Ripley, B. D.},
  isbn = {0-471-81884-4},
  publisher = {John Wiley \& Sons, Inc.},
  title = {Stochastic simulation},
  username = {dalbem},
  year = 1987
}

@article{MLMCMC,
author = {Dodwell, T. J. and Ketelsen, C. and Scheichl, R. and Teckentrup, A. L.},
title = {A Hierarchical Multilevel {M}arkov Chain {M}onte {C}arlo Algorithm with  Applications to Uncertainty Quantification in Subsurface Flow},
journal = {SIAM/ASA Journal on Uncertainty Quantification},
volume = {3},
number = {1},
pages = {1075-1108},
year = {2015},
doi = {10.1137/130915005},

URL = { 
        https://doi.org/10.1137/130915005
        },
eprint = { 
        https://doi.org/10.1137/130915005
        }
}

@misc{AEM,
author = {Cui, Tiangang and Fox, Colin and O'Sullivan, Michael},
year = {2012},
month = {12},
pages = {},
title = {Adaptive Error Modelling in {MCMC} Sampling for Large Scale Inverse Problems}
}

@article{HofGel14,
	author = {Matthew D. Hoffman and Andrew Gelman},
	title = {The No-{U}-Turn Sampler: Adaptively Setting Path Lengths in {H}amiltonian {M}onte {C}arlo},
	journal = {Journal of Machine Learning Research},
	year = 2014,
	volume = 15,
	pages = {1593--1623}
}

@article{Duaetal87,
	author = {Simon Duane and A.D. Kennedy and Brian J. Pendleton and Duncan Roweth},
	title = {Hybrid {M}onte {C}arlo},
	journal = {Physics Letters {B}},
	year = 1987,
	month = sep,
	pages = {216--222},
	volume = 195,
	number = 2
}

@book{Nea96,
	author = {Radford M. Neal},
	title = {{B}ayesian Learning for Neural Networks},
	publisher = {Springer New York, NY},
	year = 1996,
}

@ARTICLE{Ivanov:2020,
       author = {{Ivanov}, Mikhail M. and {Simonovi{\'c}}, Marko and {Zaldarriaga}, Matias},
        title = "{Cosmological parameters from the BOSS galaxy power spectrum}",
      journal = {\jcap},
         year = 2020,
        month = may,
       volume = {2020},
       number = {5},
          eid = {042},
        pages = {042},
          doi = {10.1088/1475-7516/2020/05/042},
archivePrefix = {arXiv},
       eprint = {1909.05277},
 primaryClass = {astro-ph.CO},
       adsurl = {https://ui.adsabs.harvard.edu/abs/2020JCAP...05..042I}
}

@ARTICLE{BOSS:2017,
        author = {{Alam}, Shadab and {Ata}, Metin and {Bailey}, Stephen and {Beutler}, Florian and {Bizyaev} and others},       
        fullauthor = {{Alam}, Shadab and {Ata}, Metin and {Bailey}, Stephen and {Beutler}, Florian and {Bizyaev}, Dmitry and {Blazek}, Jonathan A. and {Bolton}, Adam S. and {Brownstein}, Joel R. and {Burden}, Angela and {Chuang}, Chia-Hsun and {Comparat}, Johan and {Cuesta}, Antonio J. and {Dawson}, Kyle S. and {Eisenstein}, Daniel J. and {Escoffier}, Stephanie and {Gil-Mar{\'\i}n}, H{\'e}ctor and {Grieb}, Jan Niklas and {Hand}, Nick and {Ho}, Shirley and {Kinemuchi}, Karen and {Kirkby}, David and {Kitaura}, Francisco and {Malanushenko}, Elena and {Malanushenko}, Viktor and {Maraston}, Claudia and {McBride}, Cameron K. and {Nichol}, Robert C. and {Olmstead}, Matthew D. and {Oravetz}, Daniel and {Padmanabhan}, Nikhil and {Palanque-Delabrouille}, Nathalie and {Pan}, Kaike and {Pellejero-Ibanez}, Marcos and {Percival}, Will J. and {Petitjean}, Patrick and {Prada}, Francisco and {Price-Whelan}, Adrian M. and {Reid}, Beth A. and {Rodr{\'\i}guez-Torres}, Sergio A. and {Roe}, Natalie A. and {Ross}, Ashley J. and {Ross}, Nicholas P. and {Rossi}, Graziano and {Rubi{\~n}o-Mart{\'\i}n}, Jose Alberto and {Saito}, Shun and {Salazar-Albornoz}, Salvador and {Samushia}, Lado and {S{\'a}nchez}, Ariel G. and {Satpathy}, Siddharth and {Schlegel}, David J. and {Schneider}, Donald P. and {Sc{\'o}ccola}, Claudia G. and {Seo}, Hee-Jong and {Sheldon}, Erin S. and {Simmons}, Audrey and {Slosar}, An{\v{z}}e and {Strauss}, Michael A. and {Swanson}, Molly E.~C. and {Thomas}, Daniel and {Tinker}, Jeremy L. and {Tojeiro}, Rita and {Maga{\~n}a}, Mariana Vargas and {Vazquez}, Jose Alberto and {Verde}, Licia and {Wake}, David A. and {Wang}, Yuting and {Weinberg}, David H. and {White}, Martin and {Wood-Vasey}, W. Michael and {Y{\`e}che}, Christophe and {Zehavi}, Idit and {Zhai}, Zhongxu and {Zhao}, Gong-Bo},

        title = "{The clustering of galaxies in the completed SDSS-III Baryon Oscillation Spectroscopic Survey: cosmological analysis of the DR12 galaxy sample}",
      journal = {\mnras},
         year = 2017,
        month = sep,
       volume = {470},
       number = {3},
        pages = {2617-2652},
          doi = {10.1093/mnras/stx721},
archivePrefix = {arXiv},
       eprint = {1607.03155},
 primaryClass = {astro-ph.CO},
       adsurl = {https://ui.adsabs.harvard.edu/abs/2017MNRAS.470.2617A}
}

@ARTICLE{BOSS:2013,
        author = {{Dawson}, Kyle S. and {Schlegel}, David J. and {Ahn}, Christopher P. and {Anderson}, Scott F. and {Aubourg} and others},
       fullauthor = {{Dawson}, Kyle S. and {Schlegel}, David J. and {Ahn}, Christopher P. and {Anderson}, Scott F. and {Aubourg}, {\'E}ric and {Bailey}, Stephen and {Barkhouser}, Robert H. and {Bautista}, Julian E. and {Beifiori}, Alessandra and {Berlind}, Andreas A. and {Bhardwaj}, Vaishali and {Bizyaev}, Dmitry and {Blake}, Cullen H. and {Blanton}, Michael R. and {Blomqvist}, Michael and {Bolton}, Adam S. and {Borde}, Arnaud and {Bovy}, Jo and {Brandt}, W.~N. and {Brewington}, Howard and {Brinkmann}, Jon and {Brown}, Peter J. and {Brownstein}, Joel R. and {Bundy}, Kevin and {Busca}, N.~G. and {Carithers}, William and {Carnero}, Aurelio R. and {Carr}, Michael A. and {Chen}, Yanmei and {Comparat}, Johan and {Connolly}, Natalia and {Cope}, Frances and {Croft}, Rupert A.~C. and {Cuesta}, Antonio J. and {da Costa}, Luiz N. and {Davenport}, James R.~A. and {Delubac}, Timoth{\'e}e and {de Putter}, Roland and {Dhital}, Saurav and {Ealet}, Anne and {Ebelke}, Garrett L. and {Eisenstein}, Daniel J. and {Escoffier}, S. and {Fan}, Xiaohui and {Filiz Ak}, N. and {Finley}, Hayley and {Font-Ribera}, Andreu and {G{\'e}nova-Santos}, R. and {Gunn}, James E. and {Guo}, Hong and {Haggard}, Daryl and {Hall}, Patrick B. and {Hamilton}, Jean-Christophe and {Harris}, Ben and {Harris}, David W. and {Ho}, Shirley and {Hogg}, David W. and {Holder}, Diana and {Honscheid}, Klaus and {Huehnerhoff}, Joe and {Jordan}, Beatrice and {Jordan}, Wendell P. and {Kauffmann}, Guinevere and {Kazin}, Eyal A. and {Kirkby}, David and {Klaene}, Mark A. and {Kneib}, Jean-Paul and {Le Goff}, Jean-Marc and {Lee}, Khee-Gan and {Long}, Daniel C. and {Loomis}, Craig P. and {Lundgren}, Britt and {Lupton}, Robert H. and {Maia}, Marcio A.~G. and {Makler}, Martin and {Malanushenko}, Elena and {Malanushenko}, Viktor and {Mandelbaum}, Rachel and {Manera}, Marc and {Maraston}, Claudia and {Margala}, Daniel and {Masters}, Karen L. and {McBride}, Cameron K. and {McDonald}, Patrick and {McGreer}, Ian D. and {McMahon}, Richard G. and {Mena}, Olga and {Miralda-Escud{\'e}}, Jordi and {Montero-Dorta}, Antonio D. and {Montesano}, Francesco and {Muna}, Demitri and {Myers}, Adam D. and {Naugle}, Tracy and {Nichol}, Robert C. and {Noterdaeme}, Pasquier and {Nuza}, Sebasti{\'a}n E. and {Olmstead}, Matthew D. and {Oravetz}, Audrey and {Oravetz}, Daniel J. and {Owen}, Russell and {Padmanabhan}, Nikhil and {Palanque-Delabrouille}, Nathalie and {Pan}, Kaike and {Parejko}, John K. and {P{\^a}ris}, Isabelle and {Percival}, Will J. and {P{\'e}rez-Fournon}, Ismael and {P{\'e}rez-R{\`a}fols}, Ignasi and {Petitjean}, Patrick and {Pfaffenberger}, Robert and {Pforr}, Janine and {Pieri}, Matthew M. and {Prada}, Francisco and {Price-Whelan}, Adrian M. and {Raddick}, M. Jordan and {Rebolo}, Rafael and {Rich}, James and {Richards}, Gordon T. and {Rockosi}, Constance M. and {Roe}, Natalie A. and {Ross}, Ashley J. and {Ross}, Nicholas P. and {Rossi}, Graziano and {Rubi{\~n}o-Martin}, J.~A. and {Samushia}, Lado and {S{\'a}nchez}, Ariel G. and {Sayres}, Conor and {Schmidt}, Sarah J. and {Schneider}, Donald P. and {Sc{\'o}ccola}, C.~G. and {Seo}, Hee-Jong and {Shelden}, Alaina and {Sheldon}, Erin and {Shen}, Yue and {Shu}, Yiping and {Slosar}, An{\v{z}}e and {Smee}, Stephen A. and {Snedden}, Stephanie A. and {Stauffer}, Fritz and {Steele}, Oliver and {Strauss}, Michael A. and {Streblyanska}, Alina and {Suzuki}, Nao and {Swanson}, Molly E.~C. and {Tal}, Tomer and {Tanaka}, Masayuki and {Thomas}, Daniel and {Tinker}, Jeremy L. and {Tojeiro}, Rita and {Tremonti}, Christy A. and {Vargas Maga{\~n}a}, M. and {Verde}, Licia and {Viel}, Matteo and {Wake}, David A. and {Watson}, Mike and {Weaver}, Benjamin A. and {Weinberg}, David H. and {Weiner}, Benjamin J. and {West}, Andrew A. and {White}, Martin and {Wood-Vasey}, W.~M. and {Yeche}, Christophe and {Zehavi}, Idit and {Zhao}, Gong-Bo and {Zheng}, Zheng},
        title = "{The Baryon Oscillation Spectroscopic Survey of SDSS-III}",
      journal = {\aj},
         year = 2013,
        month = jan,
       volume = {145},
       number = {1},
          eid = {10},
        pages = {10},
          doi = {10.1088/0004-6256/145/1/10},
archivePrefix = {arXiv},
       eprint = {1208.0022},
 primaryClass = {astro-ph.CO},
       adsurl = {https://ui.adsabs.harvard.edu/abs/2013AJ....145...10D}
}

@ARTICLE{Feng:2016,
       author = {{Feng}, Yu and {Chu}, Man-Yat and {Seljak}, Uro{\v{s}} and {McDonald}, Patrick},
        title = "{FASTPM: a new scheme for fast simulations of dark matter and haloes}",
      journal = {\mnras},
         year = 2016,
        month = dec,
       volume = {463},
       number = {3},
        pages = {2273-2286},
          doi = {10.1093/mnras/stw2123},
archivePrefix = {arXiv},
       eprint = {1603.00476},
 primaryClass = {astro-ph.CO},
       adsurl = {https://ui.adsabs.harvard.edu/abs/2016MNRAS.463.2273F}
}

@BOOK{Peebles:1980,
       author = {{Peebles}, P.~J.~E.},
        title = "{The large-scale structure of the universe}",
         year = 1980,
       adsurl = {https://ui.adsabs.harvard.edu/abs/1980lssu.book.....P}
}

@ARTICLE{Spergel:2013,
        author = {{Spergel}, D. and {Gehrels}, N. and {Breckinridge}, J. and {Donahue}, M. and {Dressler}, A. and {Gaudi}, B.~S. and others },
       fullauthor = {{Spergel}, D. and {Gehrels}, N. and {Breckinridge}, J. and {Donahue}, M. and {Dressler}, A. and {Gaudi}, B.~S. and {Greene}, T. and {Guyon}, O. and {Hirata}, C. and {Kalirai}, J. and {Kasdin}, N.~J. and {Moos}, W. and {Perlmutter}, S. and {Postman}, M. and {Rauscher}, B. and {Rhodes}, J. and {Wang}, Y. and {Weinberg}, D. and {Centrella}, J. and {Traub}, W. and {Baltay}, C. and {Colbert}, J. and {Bennett}, D. and {Kiessling}, A. and {Macintosh}, B. and {Merten}, J. and {Mortonson}, M. and {Penny}, M. and {Rozo}, E. and {Savransky}, D. and {Stapelfeldt}, K. and {Zu}, Y. and {Baker}, C. and {Cheng}, E. and {Content}, D. and {Dooley}, J. and {Foote}, M. and {Goullioud}, R. and {Grady}, K. and {Jackson}, C. and {Kruk}, J. and {Levine}, M. and {Melton}, M. and {Peddie}, C. and {Ruffa}, J. and {Shaklan}, S.},
        title = "{Wide-Field InfraRed Survey Telescope-Astrophysics Focused Telescope Assets WFIRST-AFTA Final Report}",
      journal = {arXiv e-prints},
         year = 2013,
        month = may,
          eid = {arXiv:1305.5422},
        pages = {arXiv:1305.5422},
archivePrefix = {arXiv},
       eprint = {1305.5422},
 primaryClass = {astro-ph.IM},
       adsurl = {https://ui.adsabs.harvard.edu/abs/2013arXiv1305.5422S}
}

@ARTICLE{Euclid:2018,
author = {{Amendola}, Luca and {Appleby}, Stephen and {Avgoustidis}, Anastasios and
         {Bacon}, David and {Baker}, Tessa and others},
       fullauthor = {{Amendola}, Luca and {Appleby}, Stephen and {Avgoustidis}, Anastasios and
         {Bacon}, David and {Baker}, Tessa and {Baldi}, Marco and
         {Bartolo}, Nicola and {Blanchard}, Alain and {Bonvin}, Camille and
         {Borgani}, Stefano and {Branchini}, Enzo and {Burrage}, Clare and
         {Camera}, Stefano and {Carbone}, Carmelita and {Casarini}, Luciano and
         {Cropper}, Mark and {de Rham}, Claudia and {Dietrich}, J{\"o}rg P. and
         {Di Porto}, Cinzia and {Durrer}, Ruth and {Ealet}, Anne and
         {Ferreira}, Pedro G. and {Finelli}, Fabio and
         {Garc{\'\i}a-Bellido}, Juan and {Giannantonio}, Tommaso and
         {Guzzo}, Luigi and {Heavens}, Alan and {Heisenberg}, Lavinia and
         {Heymans}, Catherine and {Hoekstra}, Henk and {Hollenstein}, Lukas and
         {Holmes}, Rory and {Hwang}, Zhiqi and {Jahnke}, Knud and
         {Kitching}, Thomas D. and {Koivisto}, Tomi and {Kunz}, Martin and
         {La Vacca}, Giuseppe and {Linder}, Eric and {March}, Marisa and
         {Marra}, Valerio and {Martins}, Carlos and {Majerotto}, Elisabetta and
         {Markovic}, Dida and {Marsh}, David and {Marulli}, Federico and
         {Massey}, Richard and {Mellier}, Yannick and {Montanari}, Francesco and
         {Mota}, David F. and {Nunes}, Nelson J. and {Percival}, Will and
         {Pettorino}, Valeria and {Porciani}, Cristiano and
         {Quercellini}, Claudia and {Read}, Justin and {Rinaldi}, Massimiliano and
         {Sapone}, Domenico and {Sawicki}, Ignacy and {Scaramella}, Roberto and
         {Skordis}, Constantinos and {Simpson}, Fergus and {Taylor}, Andy and
         {Thomas}, Shaun and {Trotta}, Roberto and {Verde}, Licia and
         {Vernizzi}, Filippo and {Vollmer}, Adrian and {Wang}, Yun and
         {Weller}, Jochen and {Zlosnik}, Tom},
        title = "{Cosmology and fundamental physics with the Euclid satellite}",
      journal = {Living Reviews in Relativity},
         year = 2018,
        month = apr,
       volume = {21},
       number = {1},
          eid = {2},
        pages = {2},
          doi = {10.1007/s41114-017-0010-3},
archivePrefix = {arXiv},
       eprint = {1606.00180},
 primaryClass = {astro-ph.CO},
       adsurl = {https://ui.adsabs.harvard.edu/abs/2018LRR....21....2A}
}

@article{Salvatier2016,
  doi = {10.7717/peerj-cs.55},
  url = {https://doi.org/10.7717/peerj-cs.55},
  year  = {2016},
  month = {apr},
  publisher = {{PeerJ}},
  volume = {2},
  pages = {e55},
  author = {John Salvatier and Thomas V. Wiecki and Christopher Fonnesbeck},
  title = {Probabilistic programming in Python using {PyMC}3},
  journal = {{PeerJ} Computer Science}
}

@inproceedings{Fox1997SamplingCI,
  title={Sampling Conductivity Images via {MCMC}},
  author={Colin Fox and Geoff Nicholls},
  booktitle={The Art and Science of {B}ayesian Image Analysis},
  pages = {91--100},
  year={1997}
}

@inproceedings{MLMC,
author = {Heinrich, Stefan},
title = {Multilevel {M}onte {C}arlo Methods},
year = {2001},
isbn = {3540430431},
publisher = {Springer-Verlag},
address = {Berlin, Heidelberg},
booktitle = {Proceedings of the Third International Conference on Large-Scale Scientific Computing-Revised Papers},
pages = {58–67},
numpages = {10},
series = {LSSC '01}
}

@article{Giles-MLMC,
 ISSN = {0030364X, 15265463},
 URL = {http://www.jstor.org/stable/25147215},
 author = {Michael B. Giles},
 journal = {Operations Research},
 number = {3},
 pages = {607--617},
 publisher = {INFORMS},
 title = {Multilevel {M}onte {C}arlo Path Simulation},
 urldate = {2022-08-12},
 volume = {56},
 year = {2008}
}

@article{KAIPIO2007493,
title = {Statistical inverse problems: Discretization, model reduction and inverse crimes},
journal = {Journal of Computational and Applied Mathematics},
volume = {198},
number = {2},
pages = {493-504},
year = {2007},
note = {Special Issue: Applied Computational Inverse Problems},
issn = {0377-0427},
doi = {https://doi.org/10.1016/j.cam.2005.09.027},
url = {https://www.sciencedirect.com/science/article/pii/S0377042705007296},
author = {Jari Kaipio and Erkki Somersalo}
}

@article{AEM-summary,
 title = "A posteriori stochastic correction of reduced models in delayed-acceptance {MCMC}, with application to multiphase subsurface inverse problems",
author = "Tiangang Cui and Colin Fox and O'Sullivan, {Michael J.}",
year = "2019",
month = jun,
day = "8",
doi = "10.1002/nme.6028",
language = "English",
volume = "118",
pages = "578--605",
journal = "International Journal for Numerical Methods in Engineering",
issn = "0029-5981",
publisher = "Wiley-Blackwell",
number = "10",

}

@article{Tierney1999SomeAM,
  title={Some adaptive {M}onte {C}arlo methods for {B}ayesian inference.},
  author={Luke Tierney and Antonietta Mira},
  journal={Statistics in medicine},
  year={1999},
  volume={18 17-18},
  pages={
          2507-15
        }
}

@article{reversiblejumpreject,
 ISSN = {00063444},
 URL = {http://www.jstor.org/stable/2673700},
 author = {Peter J. Green and Antonietta Mira},
 journal = {Biometrika},
 number = {4},
 pages = {1035--1053},
 publisher = {[Oxford University Press, Biometrika Trust]},
 title = {Delayed Rejection in Reversible Jump {M}etropolis-{H}astings},
 urldate = {2022-08-12},
 volume = {88},
 year = {2001}
}

@article{multifidelitysurvey,
title = "Survey of multifidelity methods in uncertainty propagation, inference, and optimization",
author = "Benjamin Peherstorfer and Karen Willcox and Max Gunzburger",
year = "2018",
doi = "10.1137/16M1082469",
language = "English (US)",
volume = "60",
pages = "550--591",
journal = "SIAM Review",
issn = "0036-1445",
publisher = "Society for Industrial and Applied Mathematics Publications",
number = "3",

}

@book{Hairer,
  address = {Berlin},
  author = {Hairer, E. and N{\o}rsett, S.P. and Wanner, G.},
  edition = {Second},
  publisher = {Springer},
  title = {Solving Ordinary Differential 
 Equations {I} Nonstiff problems},
  year = 2000
}

@article{Roberts_2004,
	doi = {10.1214/154957804100000024},
  
	url = {https://doi.org/10.1214%2F154957804100000024},
  
	year = 2004,
	month = {jan},
  
	publisher = {Institute of Mathematical Statistics},
  
	volume = {1},
  
	number = {none},
  
	author = {Gareth O. Roberts and Jeffrey S. Rosenthal},
  
	title = {General state space {M}arkov chains and {MCMC} algorithms},
  
	journal = {Probability Surveys}
}

@article{Ivanov_2020,
	doi = {10.1088/1475-7516/2020/05/042},
  
	url = {https://doi.org/10.1088%2F1475-7516%2F2020%2F05%2F042},
  
	year = 2020,
	month = {may},
  
	publisher = {{IOP} Publishing},
  
	volume = {2020},
  
	number = {05},
  
	pages = {042--042},
  
	author = {Mikhail M. Ivanov and Marko Simonovi{\'{c}
} and Matias Zaldarriaga},
  
	title = {Cosmological parameters from the {BOSS} galaxy power spectrum},
  
	journal = {Journal of Cosmology and Astroparticle Physics}
}

@article{plank_2020,
	doi = {10.1051/0004-6361/201833910},
  
	url = {https://doi.org/10.1051%2F0004-6361%2F201833910},
  
	year = 2020,
	month = {sep},
  
	publisher = {{EDP} Sciences},
  
	volume = {641},
  
	pages = {A6},
    author = {N. Aghanim and Y. Akrami and M. Ashdown and J. Aumont and C. Baccigalupi and others},
	fullauthor = { N. Aghanim and Y. Akrami and M. Ashdown and J. Aumont and C. Baccigalupi and M. Ballardini and A. J. Banday and R. B. Barreiro and N. Bartolo and S. Basak and R. Battye and K. Benabed and J.-P. Bernard and M. Bersanelli and P. Bielewicz and J. J. Bock and J. R. Bond and J. Borrill and F. R. Bouchet and F. Boulanger and M. Bucher and C. Burigana and R. C. Butler and E. Calabrese and J.-F. Cardoso and J. Carron and A. Challinor and H. C. Chiang and J. Chluba and L. P. L. Colombo and C. Combet and D. Contreras and B. P. Crill and F. Cuttaia and P. de Bernardis and G. de Zotti and J. Delabrouille and J.-M. Delouis and E. Di Valentino and J. M. Diego and O. Dor{\'{e}
} and M. Douspis and A. Ducout and X. Dupac and S. Dusini and G. Efstathiou and F. Elsner and T. A. En{\ss}lin and H. K. Eriksen and Y. Fantaye and M. Farhang and J. Fergusson and R. Fernandez-Cobos and F. Finelli and F. Forastieri and M. Frailis and A. A. Fraisse and E. Franceschi and A. Frolov and S. Galeotta and S. Galli and K. Ganga and R. T. G{\'{e}}nova-Santos and M. Gerbino and T. Ghosh and J. Gonz{\'{a}}lez-Nuevo and K. M. G{\'{o}}rski and S. Gratton and A. Gruppuso and J. E. Gudmundsson and J. Hamann and W. Handley and F. K. Hansen and D. Herranz and S. R. Hildebrandt and E. Hivon and Z. Huang and A. H. Jaffe and W. C. Jones and A. Karakci and E. Keihänen and R. Keskitalo and K. Kiiveri and J. Kim and T. S. Kisner and L. Knox and N. Krachmalnicoff and M. Kunz and H. Kurki-Suonio and G. Lagache and J.-M. Lamarre and A. Lasenby and M. Lattanzi and C. R. Lawrence and M. Le Jeune and P. Lemos and J. Lesgourgues and F. Levrier and A. Lewis and M. Liguori and P. B. Lilje and M. Lilley and V. Lindholm and M. L{\'{o}}pez-Caniego and P. M. Lubin and Y.-Z. Ma and J. F. Mac{\'{\i}}as-P{\'{e}}rez and G. Maggio and D. Maino and N. Mandolesi and A. Mangilli and A. Marcos-Caballero and M. Maris and P. G. Martin and M. Martinelli and E. Mart{\'{\i}}nez-Gonz{\'{a}}lez and S. Matarrese and N. Mauri and J. D. McEwen and P. R. Meinhold and A. Melchiorri and A. Mennella and M. Migliaccio and M. Millea and S. Mitra and M.-A. Miville-Desch{\^{e}}nes and D. Molinari and L. Montier and G. Morgante and A. Moss and P. Natoli and H. U. N{\o}rgaard-Nielsen and L. Pagano and D. Paoletti and B. Partridge and G. Patanchon and H. V. Peiris and F. Perrotta and V. Pettorino and F. Piacentini and L. Polastri and G. Polenta and J.-L. Puget and J. P. Rachen and M. Reinecke and M. Remazeilles and A. Renzi and G. Rocha and C. Rosset and G. Roudier and J. A. Rubi{\~{n}}o-Mart{\'{\i}}n and B. Ruiz-Granados and L. Salvati and M. Sandri and M. Savelainen and D. Scott and E. P. S. Shellard and C. Sirignano and G. Sirri and L. D. Spencer and R. Sunyaev and A.-S. Suur-Uski and J. A. Tauber and D. Tavagnacco and M. Tenti and L. Toffolatti and M. Tomasi and T. Trombetti and L. Valenziano and J. Valiviita and B. Van Tent and L. Vibert and P. Vielva and F. Villa and N. Vittorio and B. D. Wandelt and I. K. Wehus and M. White and S. D. M. White and A. Zacchei and A. Zonca},
  
	problematictitle = {$\less$i$\greater$Planck$\less$/i$\greater$2018 results},
	title = {{Planck} 2018 results},
  
	problematicjournal = {Astronomy {\&}amp$\mathsemicolon$ Astrophysics},
	journal = {Astronomy \& Astrophysics}
}

@article{bj/1080222083,
author = {Heikki Haario and Eero Saksman and Johanna Tamminen},
title = {An adaptive {M}etropolis algorithm},
volume = {7},
journal = {Bernoulli},
number = {2},
publisher = {Bernoulli Society for Mathematical Statistics and Probability},
pages = {223 -- 242},
year = {2001},
}

@INCOLLECTION{Gelman1996,
  author = {Gelman, A. and Roberts, G. O. and Gilks, W. R.},
  title = {Efficient {M}etropolis jumping rules},
  booktitle = {Bayesian Statistics},
  publisher = {Oxford University Press, Oxford},
  year = {1996},
  editor = {Bernardo, J. M. and Berger, J. O. and Dawid, A. P. and Smith, A.
	F. M.},
  pages = {599-608}
}

@article{hinton2002training,
  title={Training products of experts by minimizing contrastive divergence},
  author={Hinton, Geoffrey E},
  journal={Neural computation},
  volume={14},
  number={8},
  pages={1771--1800},
  year={2002},
  publisher={MIT Press}
}

@InProceedings{pmlr-v9-sutskever10a,
  title = 	 {On the Convergence Properties of Contrastive Divergence},
  author = 	 {Sutskever, Ilya and Tieleman, Tijmen},
  booktitle = 	 {Proceedings of the Thirteenth International Conference on Artificial Intelligence and Statistics},
  pages = 	 {789--795},
  year = 	 {2010},
  editor = 	 {Teh, Yee Whye and Titterington, Mike},
  volume = 	 {9},
  series = 	 {Proceedings of Machine Learning Research},
  address = 	 {Chia Laguna Resort, Sardinia, Italy},
  month = 	 {13--15 May},
  publisher =    {PMLR},
  url = 	 {https://proceedings.mlr.press/v9/sutskever10a.html}
}

@inproceedings{NEURIPS2022_8803b9ae,
 author = {Cai, Diana and Adams, Ryan P},
 booktitle = {Advances in Neural Information Processing Systems},
 editor = {S. Koyejo and S. Mohamed and A. Agarwal and D. Belgrave and K. Cho and A. Oh},
 pages = {21654--21667},
 publisher = {Curran Associates, Inc.},
 title = {Multi-fidelity {M}onte {C}arlo: a pseudo-marginal approach},
 url = {https://proceedings.neurips.cc/paper_files/paper/2022/file/8803b9ae0b13011f28e6dd57da2ebbd8-Paper-Conference.pdf},
 volume = {35},
 year = {2022}
}

@article{Marinari,
author="E. Marinari and G. Parisi",
title="Simulated Tempering: A New {M}onte {C}arlo Scheme",
journal="Europhysics Letters (EPL)",
ISSN="0295-5075",
publisher="IOP Publishing",
year="1992",
month="07",
volume="19",
number="6",
pages="451-458",
DOI="10.1209/0295-5075/19/6/002",
URL="https://cir.nii.ac.jp/crid/1363670319736685312"
}

@article{MCSimulation,
  title = {Replica Monte Carlo Simulation of Spin-Glasses},
  author = {Swendsen, Robert H. and Wang, Jian-Sheng},
  journal = {Phys. Rev. Lett.},
  volume = {57},
  issue = {21},
  pages = {2607--2609},
  numpages = {0},
  year = {1986},
  month = {Nov},
  publisher = {American Physical Society},
  doi = {10.1103/PhysRevLett.57.2607},
  url = {https://link.aps.org/doi/10.1103/PhysRevLett.57.2607}
}

@article{Altekar_Dwarkadas_Huelsenbeck_Ronquist_2004, title={Parallel Metropolis Coupled Markov chain Monte Carlo for bayesian phylogenetic inference}, volume={20}, DOI={10.1093/bioinformatics/btg427}, number={3}, journal={Bioinformatics}, author={Altekar, Gautam and Dwarkadas, Sandhya and Huelsenbeck, John P. and Ronquist, Fredrik}, year={2004}, month={Jan}, pages={407–415}}

@article{10.1093/biomet/82.4.711,
    author = {Green, Peter J.},
    title = "{Reversible jump Markov chain Monte Carlo computation and Bayesian model determination}",
    journal = {Biometrika},
    volume = {82},
    number = {4},
    pages = {711-732},
    year = {1995},
    month = {12},
    issn = {0006-3444},
    doi = {10.1093/biomet/82.4.711},
    url = {https://doi.org/10.1093/biomet/82.4.711},
    eprint = {https://academic.oup.com/biomet/article-pdf/82/4/711/699533/82-4-711.pdf},
}

@article{Al-Awadhi,
author = {Al-Awadhi, Fahimah and Hurn, Merrilee and Jennison, Christopher},
year = {2004},
month = {08},
pages = {189-198},
title = {Improving the acceptance rate of reversible jump MCMC proposals},
volume = {69},
journal = {Statistics \& Probability Letters},
doi = {10.1016/j.spl.2004.06.025}
}

@article{smc,
author = {Liu, Jun and Chen, Rong},
year = {1998},
month = {04},
pages = {},
title = {Sequential Monte Carlo Methods for Dynamic Systems},
volume = {93},
journal = {Journal of the American Statistical Association},
doi = {10.1080/01621459.1998.10473765}
}

@Inbook{Doucet2001,
author="Doucet, Arnaud
and de Freitas, Nando
and Gordon, Neil",
editor="Doucet, Arnaud
and de Freitas, Nando
and Gordon, Neil",
title="An Introduction to Sequential Monte Carlo Methods",
bookTitle="Sequential Monte Carlo Methods in Practice",
year="2001",
publisher="Springer New York",
address="New York, NY",
pages="3--14",
isbn="978-1-4757-3437-9",
doi="10.1007/978-1-4757-3437-9_1",
url="https://doi.org/10.1007/978-1-4757-3437-9_1"
}

@article{diminishing-adap-ergodicity,
 ISSN = {00219002},
 URL = {http://www.jstor.org/stable/27595854},
 author = {Gareth O. Roberts and Jeffrey S. Rosenthal},
 journal = {Journal of Applied Probability},
 number = {2},
 pages = {458--475},
 publisher = {Applied Probability Trust},
 title = {Coupling and Ergodicity of Adaptive Markov Chain Monte Carlo Algorithms},
 urldate = {2024-10-10},
 volume = {44},
 year = {2007}
}

@article{Abdalla_2022,
   title={Cosmology intertwined: A review of the particle physics, astrophysics, and cosmology associated with the cosmological tensions and anomalies},
   volume={34},
   ISSN={2214-4048},
   url={http://dx.doi.org/10.1016/j.jheap.2022.04.002},
   DOI={10.1016/j.jheap.2022.04.002},
   journal={Journal of High Energy Astrophysics},
   publisher={Elsevier BV},
author = {Abdalla, Elcio and Abell{\'a}n, Guillermo Franco and Aboubrahim, Amin and Agnello, Adriano and Akarsu, {\"O}zg{\"u}r and others},
   fullauthor={Abdalla, Elcio and Abellán, Guillermo Franco and Aboubrahim, Amin and Agnello, Adriano and Akarsu, Özgür and Akrami, Yashar and Alestas, George and Aloni, Daniel and Amendola, Luca and Anchordoqui, Luis A. and Anderson, Richard I. and Arendse, Nikki and Asgari, Marika and Ballardini, Mario and Barger, Vernon and Basilakos, Spyros and Batista, Ronaldo C. and Battistelli, Elia S. and Battye, Richard and Benetti, Micol and Benisty, David and Berlin, Asher and de Bernardis, Paolo and Berti, Emanuele and Bidenko, Bohdan and Birrer, Simon and Blakeslee, John P. and Boddy, Kimberly K. and Bom, Clecio R. and Bonilla, Alexander and Borghi, Nicola and Bouchet, François R. and Braglia, Matteo and Buchert, Thomas and Buckley-Geer, Elizabeth and Calabrese, Erminia and Caldwell, Robert R. and Camarena, David and Capozziello, Salvatore and Casertano, Stefano and Chen, Geoff C.-F. and Chluba, Jens and Chen, Angela and Chen, Hsin-Yu and Chudaykin, Anton and Cicoli, Michele and Copi, Craig J. and Courbin, Fred and Cyr-Racine, Francis-Yan and Czerny, Bożena and Dainotti, Maria and D’Amico, Guido and Davis, Anne-Christine and de Cruz Pérez, Javier and de Haro, Jaume and Delabrouille, Jacques and Denton, Peter B. and Dhawan, Suhail and Dienes, Keith R. and Di Valentino, Eleonora and Du, Pu and Eckert, Dominique and Escamilla-Rivera, Celia and Ferté, Agnès and Finelli, Fabio and Fosalba, Pablo and Freedman, Wendy L. and Frusciante, Noemi and Gaztañaga, Enrique and Giarè, William and Giusarma, Elena and Gómez-Valent, Adrià and Handley, Will and Harrison, Ian and Hart, Luke and Hazra, Dhiraj Kumar and Heavens, Alan and Heinesen, Asta and Hildebrandt, Hendrik and Hill, J. Colin and Hogg, Natalie B. and Holz, Daniel E. and Hooper, Deanna C. and Hosseininejad, Nikoo and Huterer, Dragan and Ishak, Mustapha and Ivanov, Mikhail M. and Jaffe, Andrew H. and Jang, In Sung and Jedamzik, Karsten and Jimenez, Raul and Joseph, Melissa and Joudaki, Shahab and Kamionkowski, Marc and Karwal, Tanvi and Kazantzidis, Lavrentios and Keeley, Ryan E. and Klasen, Michael and Komatsu, Eiichiro and Koopmans, Léon V.E. and Kumar, Suresh and Lamagna, Luca and Lazkoz, Ruth and Lee, Chung-Chi and Lesgourgues, Julien and Levi Said, Jackson and Lewis, Tiffany R. and L’Huillier, Benjamin and Lucca, Matteo and Maartens, Roy and Macri, Lucas M. and Marfatia, Danny and Marra, Valerio and Martins, Carlos J.A.P. and Masi, Silvia and Matarrese, Sabino and Mazumdar, Arindam and Melchiorri, Alessandro and Mena, Olga and Mersini-Houghton, Laura and Mertens, James and Milaković, Dinko and Minami, Yuto and Miranda, Vivian and Moreno-Pulido, Cristian and Moresco, Michele and Mota, David F. and Mottola, Emil and Mozzon, Simone and Muir, Jessica and Mukherjee, Ankan and Mukherjee, Suvodip and Naselsky, Pavel and Nath, Pran and Nesseris, Savvas and Niedermann, Florian and Notari, Alessio and Nunes, Rafael C. and Ó Colgáin, Eoin and Owens, Kayla A. and Özülker, Emre and Pace, Francesco and Paliathanasis, Andronikos and Palmese, Antonella and Pan, Supriya and Paoletti, Daniela and Perez Bergliaffa, Santiago E. and Perivolaropoulos, Leandros and Pesce, Dominic W. and Pettorino, Valeria and Philcox, Oliver H.E. and Pogosian, Levon and Poulin, Vivian and Poulot, Gaspard and Raveri, Marco and Reid, Mark J. and Renzi, Fabrizio and Riess, Adam G. and Sabla, Vivian I. and Salucci, Paolo and Salzano, Vincenzo and Saridakis, Emmanuel N. and Sathyaprakash, Bangalore S. and Schmaltz, Martin and Schöneberg, Nils and Scolnic, Dan and Sen, Anjan A. and Sehgal, Neelima and Shafieloo, Arman and Sheikh-Jabbari, M.M. and Silk, Joseph and Silvestri, Alessandra and Skara, Foteini and Sloth, Martin S. and Soares-Santos, Marcelle and Solà Peracaula, Joan and Songsheng, Yu-Yang and Soriano, Jorge F. and Staicova, Denitsa and Starkman, Glenn D. and Szapudi, István and Teixeira, Elsa M. and Thomas, Brooks and Treu, Tommaso and Trott, Emery and van de Bruck, Carsten and Vazquez, J. Alberto and Verde, Licia and Visinelli, Luca and Wang, Deng and Wang, Jian-Min and Wang, Shao-Jiang and Watkins, Richard and Watson, Scott and Webb, John K. and Weiner, Neal and Weltman, Amanda and Witte, Samuel J. and Wojtak, Radosław and Yadav, Anil Kumar and Yang, Weiqiang and Zhao, Gong-Bo and Zumalacárregui, Miguel},
   year={2022},
   month=jun, pages={49–211} }

@article{doi:10.1137/22M1476770,
author = {Lykkegaard, M. B. and Dodwell, T. J. and Fox, C. and Mingas, G. and Scheichl, R.},
title = {Multilevel Delayed Acceptance MCMC},
journal = {SIAM/ASA Journal on Uncertainty Quantification},
volume = {11},
number = {1},
pages = {1-30},
year = {2023},
doi = {10.1137/22M1476770},

URL = { 
    
        https://doi.org/10.1137/22M1476770
    
    

},
eprint = { 
    
        https://doi.org/10.1137/22M1476770
    
    

}
}

@article{Conrad_2016,
   title={Accelerating Asymptotically Exact MCMC for Computationally Intensive Models via Local Approximations},
   volume={111},
   ISSN={1537-274X},
   url={http://dx.doi.org/10.1080/01621459.2015.1096787},
   DOI={10.1080/01621459.2015.1096787},
   number={516},
   journal={Journal of the American Statistical Association},
   publisher={Informa UK Limited},
   author={Conrad, Patrick R. and Marzouk, Youssef M. and Pillai, Natesh S. and Smith, Aaron},
   year={2016},
   month=oct, pages={1591–1607} }

@ARTICLE{978393,
  author={Higdon, D. and Lee, H. and Zhuoxin Bi},
  journal={IEEE Transactions on Signal Processing}, 
  title={A Bayesian approach to characterizing uncertainty in inverse problems using coarse and fine-scale information}, 
  year={2002},
  volume={50},
  number={2},
  pages={389-399},
  doi={10.1109/78.978393}}

@article{doi:10.1137/050628568,
author = {Efendiev, Y. and Hou, T. and Luo, W.},
title = {Preconditioning Markov Chain Monte Carlo Simulations Using Coarse-Scale Models},
journal = {SIAM Journal on Scientific Computing},
volume = {28},
number = {2},
pages = {776-803},
year = {2006},
doi = {10.1137/050628568},

URL = { 
    
        https://doi.org/10.1137/050628568
    
    

},
eprint = { 
    
        https://doi.org/10.1137/050628568
    
    

}
}

@book{Strasser+1985,
url = {https://doi.org/10.1515/9783110850826},
title = {Mathematical Theory of Statistics},
title = {Statistical Experiments and Asymptotic Decision Theory},
author = {Helmut Strasser},
publisher = {De Gruyter},
address = {Berlin, New York},
doi = {doi:10.1515/9783110850826},
isbn = {9783110850826},
year = {1985},
lastchecked = {2025-07-21}
}

@article{hoang2013complexity,
  title={Complexity analysis of accelerated MCMC methods for Bayesian inversion},
  author={Hoang, Viet Ha and Schwab, Christoph and Stuart, Andrew M},
  journal={Inverse Problems},
  volume={29},
  number={8},
  pages={085010},
  year={2013},
  publisher={IOP Publishing}
}

@article{beskos2017multilevel,
  title={Multilevel sequential monte carlo samplers},
  author={Beskos, Alexandros and Jasra, Ajay and Law, Kody and Tempone, Raul and Zhou, Yan},
  journal={Stochastic Processes and their Applications},
  volume={127},
  number={5},
  pages={1417--1440},
  year={2017},
  publisher={Elsevier}
}

@article{jasra2017multilevel,
  title={Multilevel particle filters},
  author={Jasra, Ajay and Kamatani, Kengo and Law, Kody JH and Zhou, Yan},
  journal={SIAM Journal on Numerical Analysis},
  volume={55},
  number={6},
  pages={3068--3096},
  year={2017},
  publisher={SIAM}
}

@article{hoel2016multilevel,
  title={Multilevel ensemble Kalman filtering},
  author={Hoel, H{\aa}kon and Law, Kody JH and Tempone, Raul},
  journal={SIAM Journal on Numerical Analysis},
  volume={54},
  number={3},
  pages={1813--1839},
  year={2016},
  publisher={SIAM}
}

@article{gregory2017seamless,
  title={A seamless multilevel ensemble transform particle filter},
  author={Gregory, Alastair and Cotter, Colin J},
  journal={SIAM Journal on Scientific Computing},
  volume={39},
  number={6},
  pages={A2684--A2701},
  year={2017},
  publisher={SIAM}
}

@article{gregory2016multilevel,
  title={Multilevel ensemble transform particle filtering},
  author={Gregory, Alastair and Cotter, Colin J and Reich, Sebastian},
  journal={SIAM Journal on Scientific Computing},
  volume={38},
  number={3},
  pages={A1317--A1338},
  year={2016},
  publisher={SIAM}
}

@article{jasra2018bayesian,
  title={Bayesian static parameter estimation for partially observed diffusions via multilevel Monte Carlo},
  author={Jasra, Ajay and Kamatani, Kengo and Law, Kody and Zhou, Yan},
  journal={SIAM Journal on Scientific Computing},
  volume={40},
  number={2},
  pages={A887--A902},
  year={2018},
  publisher={SIAM}
}

@article{jasra2018multi,
  title={A multi-index Markov chain Monte Carlo method},
  author={Jasra, Ajay and Kamatani, Kengo and Law, Kody JH and Zhou, Yan},
  journal={International Journal for Uncertainty Quantification},
  volume={8},
  number={1},
  year={2018},
  publisher={Begel House Inc.}
}

@article{haji2016multi,
  title={Multi-index stochastic collocation convergence rates for random PDEs with parametric regularity},
  author={Haji-Ali, Abdul-Lateef and Nobile, Fabio and Tamellini, Lorenzo and Tempone, Ra{\'u}l},
  journal={Foundations of Computational Mathematics},
  volume={16},
  number={6},
  pages={1555--1605},
  year={2016},
  publisher={Springer}
}

\end{document}